\documentclass{article}

\usepackage{microtype}
\usepackage{graphicx}
\usepackage{subfigure}
\usepackage{booktabs} 

\usepackage{hyperref}
\usepackage{colortbl}
\usepackage{aligned-overset}
\usepackage{tikz}
\usepackage{makecell}
\usepackage{amsthm}
\usepackage{bbm}
\usepackage{caption}
\usetikzlibrary{positioning, bayesnet}
\usepackage{multirow}
\usepackage{tablefootnote}
\usepackage{textcomp}
\usepackage{amsmath}
\usepackage{stfloats}

\usepackage{nicefrac}       
\usepackage{microtype}      
\usepackage{xcolor}         
\usepackage{bm}

\newcommand{\dif}{\mathrm{d}}

\usepackage[accepted]{icml2024}

\usepackage{amsmath}
\usepackage{amssymb}
\usepackage{mathtools}
\usepackage{amsthm}

\usepackage[capitalize,noabbrev]{cleveref}

\theoremstyle{plain}
\newtheorem{theorem}{Theorem}[section]

\theoremstyle{definition}

\theoremstyle{remark}
\newtheorem{remark}[theorem]{Remark}
\newtheorem{example}[theorem]{Example}
\usepackage[disable,textsize=tiny]{todonotes}

\icmltitlerunning{Reflected Flow Matching}

\begin{document}

\twocolumn[
\icmltitle{Reflected Flow Matching}



\icmlsetsymbol{equal}{*}

\begin{icmlauthorlist}
\icmlauthor{Tianyu Xie$^\ast$}{pkumath}
\icmlauthor{Yu Zhu$^\ast$}{cas,baai}
\icmlauthor{Longlin Yu$^\ast$}{pkumath}
\icmlauthor{Tong Yang}{fdu,meg}
\icmlauthor{Ziheng Cheng}{pkumath}\\
\icmlauthor{Shiyue Zhang}{pkumath}
\icmlauthor{Xiangyu Zhang}{meg}
\icmlauthor{Cheng Zhang}{pkumath,pkucss}
\end{icmlauthorlist}

\icmlaffiliation{pkumath}{School of Mathematical Sciences, Peking University, Beijing, China}
\icmlaffiliation{pkucss}{Center for Statistical Science, Peking University, Beijing, China}
\icmlaffiliation{cas}{Institute of Automation, Chinese Academy of Sciences, Beijing, China}
\icmlaffiliation{baai}{Beijing Academy of Artificial Intelligence, Beijing, China}
\icmlaffiliation{fdu}{School of Computer Science, Fudan University, Shanghai, China}
\icmlaffiliation{meg}{Megvii Technology Inc., Beijing, China}

\icmlcorrespondingauthor{Cheng Zhang}{chengzhang@math.pku.edu.cn}

\icmlkeywords{Machine Learning, ICML}
\vskip 0.3in
]



\printAffiliationsAndNotice{\icmlEqualContribution} 

\begin{abstract}
Continuous normalizing flows (CNFs) learn an ordinary differential equation to transform prior samples into data.
Flow matching (FM) has recently emerged as a simulation-free approach for training CNFs by regressing a velocity model towards the conditional velocity field.
However, on constrained domains, the learned velocity model may lead to undesirable flows that result in highly unnatural samples, e.g., oversaturated images, due to both flow matching error and simulation error.
To address this, we add a boundary constraint term to CNFs, which leads to reflected CNFs that keep trajectories within the constrained domains.
We propose reflected flow matching (RFM) to train the velocity model in reflected CNFs by matching the conditional velocity fields in a simulation-free manner, similar to the vanilla FM.
Moreover, the analytical form of conditional velocity fields in RFM avoids potentially biased approximations, making it superior to existing score-based generative models on constrained domains.
We demonstrate that RFM achieves comparable or better results on standard image benchmarks and produces high-quality class-conditioned samples under high guidance weight.
\end{abstract}

\section{Introduction}
Deep generative models, which are deep learning models designed to generate new data samples that resemble a given data set, find applications in various domains including image synthesis \citep{Arjovsky2017WGAN, ho2020denoising, dhariwal2021diffusion, song2023consistency}, text generation \citep{devlin2018bert,brown2020GPT3,zhang2022opt}, and molecular design \citep{jin2018junctiontree, xu2021geodiff, hoogeboom2022equivariant}.
While the recent advances in this domain are mostly driven by diffusion models \citep{ho2020denoising,song2020score}, a powerful alternative is flow-based model \citep{NF}, which works by learning a \textit{flow}, i.e., a function that transforms samples from a simple prior distribution into samples distributed similarly to the target data distribution.
Flows can also be implicitly defined via ordinary differential equations (ODEs), which lead to a general framework called continuous normalizing flows (CNFs) \citep{chen2018neuralode}.
Despite their flexibility, classical maximum likelihood training of CNFs can be inefficient as it requires expensive numerical ODE simulations.

Alternatively, flow matching (FM) \citep{lipman2023FM} presents an efficient and scalable training method for CNFs by regressing a learnable velocity model to the target velocity field (i.e., the drift term in ODE) that generates the data distribution.
Inspired by denoising score matching (DSM) \citep{Vincent2011}, FM decomposes the intractable target velocity field into a mixture of analytical conditional velocity fields, leading to a regression objective that is computationally stable and simulation-free.
A typical example of conditional velocity fields is the optimal transport (OT) conditional velocity field, which has a constant speed along the straight line from a prior sample to a data point (i.e., its condition).
FM has also been extended to equivariance modeling \citep{klein2023equivariantFM}, Bayesian inference \citep{wildberger2023flowSBI}, and dynamic optimal transport \citep{tong2023CFM}, etc.

Despite the superiority of FM, the learned velocity model may cause unreasonable flows, due to both the inherent flow matching error (i.e., the difference between the learned velocity model and the true velocity field) and compounded simulation error (i.e., the error introduced by discretizing the ODE during sampling), which can generate highly unnatural samples, especially for complex data distributions on constrained domains. 
For example, an RGB image is constrained to a domain $[0,255]^d$ ($d$ is the number of digits), and a boundary violation may result in collapsed and singular samples.
Recently, there has been a growing interest in generative modeling on constrained domains, driven by its broad range of applications \citep{fishman2023diffusionCD, lou2023reflectedDM, liu2023mirror}.
Within the framework of score-based generative models \citep{song2020score,ho2020denoising}, reflected diffusion models (RDMs) \citep{lou2023reflectedDM} incorporate the constraints through a reflected stochastic differential equation (SDE) that always stays within the domain.
The reverse diffusion process is then parameterized with the scores of the perturbed intermediate distributions, which can be trained via DSM.
However, the transition kernels in RDMs lack a closed-form expression, necessitating the complicated conditional score approximation that requires computing reflecting points \citep{jing2022torsionaldiffusion}  or solving partial differential equations \citep{Bortoli2022RiemannianSGM}.
These challenges introduce bias and make the application of RDMs difficult for general domains, limiting their practicality to simple constrained domains like hypercubes.
Although implicit score matching (ISM) avoids the computation of transition kernels \citep{fishman2023diffusionCD}, it is less stable than DSM and does not scale well to high-dimensional cases.

In this work, we propose reflected flow matching (RFM), which extends the training and sampling procedures of CNFs to constrained domains. 
To achieve this, we add an additional boundary constraint term to the ODEs in CNFs, which we call reflected CNFs. 
When initiated from prior distributions supported on the target constrained domain, the reflected CNFs ensure that the trajectories remain within this domain.
The velocity fields in the reflected CNFs can be trained by matching the conditional velocity fields in a simulation-free manner, akin to vanilla FM \citep{lipman2023FM}.
Moreover, these conditional velocity fields can be readily derived from corresponding conditional flows that satisfy the constraints.
This means that, unlike RDMs \citep{lou2023reflectedDM}, FM can be smoothly generalized to constrained domains with no extra computational burden.
Moreover, inherited from FM, the RFM framework permits the choice of arbitrary prior distributions on the constrained domain, which may provide extra flexibility for generative modeling.
We demonstrate the effectiveness of RFM across various generative tasks, including low-dimensional toy examples as well as unconditional and conditional image generation benchmarks.
Our code is available at \href{https://github.com/tyuxie/RFM}{\texttt{https://github.com/tyuxie/RFM}}.
\section{Background}

\subsection{Flow Matching}

A time-dependent \emph{flow} $\bm{x}_t=\bm{\phi}_t(\bm{x})$, which describes the motion from a base point $\bm{x}$, can be defined by an ordinary differential equation (ODE)
\begin{subequations}\label{eq:particle-ODE}
\arraycolsep=1.8pt
\def\arraystretch{1.5}
\begin{align}
\mathrm{d}\bm{\phi}_t(\bm{x}) &= \bm{v}_t(\bm{\phi}_t(\bm{x}))\mathrm{d}t,\\
\bm{\phi}_0(\bm{x}) &=\bm{x},
\end{align}
\end{subequations}
for $t\in[0,1]$, where $\bm{v}_t(\cdot)\in\mathbb{R}^d$ is called the \emph{velocity field}. 
Given a density $p_0(\bm{x})$ at $t=0$, the velocity field induces a \emph{probability path} $p_t(\bm{x}): [0,1]\times \mathbb{R}^d\to\mathbb{R}$, where $p_t$ is the probability density function (PDF) of the points from $p_0(\bm{x})$ transported along $\bm{v}_t$ from time $0$ to time $t$, and it is characterized by the well-known continuity equation:
\begin{equation}
\frac{\partial }{\partial t}p_t(\bm{x}) = -\nabla\cdot(p_t(\bm{x})\bm{v}_t(\bm{x})), \ \bm{x}\in\mathbb{R}^d.
\end{equation}
\citet{chen2018neuralode} models the velocity field by a neural network $\bm{v}_{\bm{\theta}}(\bm{x},t)$ with parameters $\bm{\theta}$ to generate a learnable model of $\bm{\phi}_t$, called a continuous normalizing flow (CNF).
Given a target probability path $p_t(\bm{x})$ that connects a simple prior distribution $p_0(\bm{x})$ and a complicated target distribution $p_{1}(\bm{x})$ which is close to the data distribution $p_{\mathrm{data}}(\bm{x})$,
the main idea of flow matching (FM) \citep{lipman2023FM} is to regress the velocity model $\bm{v}_{\bm{\theta}}(\bm{x},t)$ to the target velocity field $\bm{v}_t(\bm{x})$ by minimizing the FM objective
\begin{equation}
\mathcal{L}_{\textrm{FM}}(\bm{\theta}) = \int_0^1\mathbb{E}_{p_t(\bm{x})}\|v_{\bm{\theta}}(\bm{x},t)-\bm{v}_t(\bm{x})\|^2\dif t.
\end{equation}
One way to construct such a target probability path is to create per-sample \emph{conditional probability paths} $p_t(\bm{x}|\bm{x}_1)$ that satisfy $p_0(\bm{x}|\bm{x}_1)=p_0(\bm{x}),p_1(\bm{x}|\bm{x}_1)\approx \delta(\bm{x}_1)$, and marginalize them over the data distribution as follows
\begin{equation}
p_t(\bm{x})=\int_{\mathbb{R}^d}p_t(\bm{x}|\bm{x}_1)p_{\mathrm{data}}(\bm{x}_1)\dif \bm{x}_1.
\end{equation}
\citet{lipman2023FM} then shows that the corresponding velocity field $\bm{v}_t(\bm{x})$ takes the following form 
\begin{equation}\label{eq:fm-marginal-velocity-field}
\bm{v}_t(\bm{x})=\int_{\mathbb{R}^d}\bm{v}_t(\bm{x}|\bm{x}_1)\frac{p_t(\bm{x}|\bm{x}_1)p_{\mathrm{data}}(\bm{x}_1)}{p_t(\bm{x})}\dif\bm{x}_1,
\end{equation}
where $\bm{v}_t(\bm{x}|\bm{x}_1)$ is the conditional velocity field that generates $p_t(\bm{x}|\bm{x}_1)$.
The velocity model $\bm{v}_{\bm{\theta}}(\bm{x},t)$, therefore, can be trained by minimizing the conditional flow matching (CFM) objective
\begin{equation}\label{eq:cfm}
\mathcal{L}_{\textrm{CFM}}(\bm{\theta})=\int_{0}^1 \mathbb{E}_{ p_t(\bm{x}|\bm{x}_1)p_{\mathrm{data}}(\bm{x}_1)}\|\bm{v}_{\bm{\theta}}(\bm{x}, t)-\bm{v}_t(\bm{x}|\bm{x}_1)\|^2\mathrm{d}t.
\end{equation}

Concretely, \citet{lipman2023FM} considers the Gaussian conditional probability paths, which can be induced by a flow conditioned on $\bm{x}_1$ as
\begin{equation}
\bm{\phi}_t(\bm{x}|\bm{x}_1) = \sigma_t(\bm{x}_1)\bm{x} + \bm{\mu}_t(\bm{x}_1),
\end{equation}
where $\bm{x}$ follows a standard Gaussian distribution.
Consequently, the conditional velocity field takes the form
\begin{equation}\label{eq:fm-conditional-velocity-field}
\bm{v}_{t}(\bm{x}|\bm{x}_1)=\frac{\sigma_t'(\bm{x}_1)}{\sigma_t(\bm{x}_1)}\left(\bm{x}-\bm{\mu}_t(\bm{x}_1)\right)+\bm{\mu}_t'(\bm{x}_1).
\end{equation}
A typical example is the optimal transport (OT) conditional velocity field, where $\sigma_t(\bm{x}_1) = 1-(1-\sigma_{\mathrm{min}})t$, $\bm{\mu}_t(\bm{x}_1)=t\bm{x}_1$, and $\sigma_{\mathrm{min}}$ is a sufficiently small value.

\subsection{Reflected Diffusion Models}
For generative modeling on constrained domains, reflected diffusion models (RDMs) \citep{lou2023reflectedDM} use reflected stochastic differential equations (SDEs) that perform reflection operations at the boundary to keep samples inside the target domain.
Let the support of the data distribution $q_{0}(\cdot)$ be $\Omega$, which is assumed to be connected and compact with a nonempty interior.
Concretely, the forward process in an RDM is described by the following reflected SDE
\begin{equation}\label{eq:sde}
\dif \bm{u}_t = \bm{f}(\bm{u}_t,t) \dif t + g(t)\dif \bm{B}_t + \dif \bm{L}_t, \ \bm{u}_0\sim q_0(\cdot),
\end{equation}
where $t\in[0,L]$, $\bm{f}(\bm{u}_t,t)$ and $g(t)$ are the drift term and diffusion term as in the classical diffusion models \citep{song2020score}, and $\bm{L}_t$ is an additional boundary constraint that intuitively forces the particle to stay inside $\Omega$.
When $\bm{u}_t$ hits the boundary $\partial\Omega$, $\bm{L}_t$ neutralizes the outward-normal velocity. 
To generate samples from the data distribution, one can simulate the reverse reflected SDE \citep{cattiaux1988reversedSDE, williams1987reversedSDE}
\begin{equation}\label{eq:reverse-sde}
\small
\begin{split}
  \dif \bm{u}_t = \left[\bm{f}(\bm{u}_t,t) -\frac{1}{2}(1+\lambda^2)g^2(t) \nabla\log q_t(\bm{u}_t)\right] \dif t \\+ \lambda g(t)\dif \bar{\bm{B}}_t +\dif\bm{L}_t, \ \bm{u}_L\sim q_L(\cdot),
\end{split}
\end{equation}
where $q_L(\cdot)=\mathrm{uniform}(\Omega)$ under some special conditions, $\bar{\bm{B}}_t$ is a standard Brownian motion when time flows back from $L$ to $0$ and $\nabla\log q_t(\bm{u}_t)$ is the score function at time $t$.

Although denoising score matching (DSM) \citep{Vincent2011, song2020score} succeeds in estimating the intractable score function $\nabla\log q_t(\bm{u}_t)$, the required conditional score function $\nabla_{\bm{u}_t}\log q_t(\bm{u}_t|\bm{u}_0)$ is no longer analytically available for the forward process in equation \eqref{eq:sde}.
To estimate $\nabla_{\bm{u}_t}\log q_t(\bm{u}_t|\bm{u}_0)$, RDMs rely on approximation methods based on sum of Gaussians \citep{jing2022torsionaldiffusion} or Laplacian eigenfunctions \citep{Bortoli2022RiemannianSGM}.
However, both of them can be domain-dependent and require truncating an infinite series of functions.
Besides, although the noisy training data in DSM can be obtained in a simulation-free manner, the required geometric techniques may be restrictive in practice.

\subsection{Classifier-free Guidance}\label{sec:guidance-diffusion}
Diffusion models can be controlled to generate guided samples from $\tilde{q}_t(\bm{u}|c)\propto q_t(c|\bm{u})^w q_t(\bm{u})$ where $c$ is a condition (e.g., a class label) and $w$ is the guidance weight \citep{ho2022classifier}.
The score function of $\tilde{q}_t(\bm{u}|c)$ satisfies
\begin{equation}\label{eq:conditional-score}
\nabla \log\tilde{q}_t(\bm{u}|c) = w\nabla\log q_t(c|\bm{u}) + \nabla\log q_t(\bm{u}).
\end{equation}
With the Bayes formula $q_t(c|\bm{u})=\frac{q_t(\bm{u}|c)q_t(c)}{q_t(\bm{u})}$ and a conditional score model $\bm{s}_\phi(\bm{u},t,c)\approx \nabla\log q_t(\bm{u}|c)$, the score function of $\tilde{q}_t(\bm{u}|c)$ can be estimated by
\begin{equation}\label{eq:score-classifier}
\tilde{\bm{s}}_\phi(\bm{u},t,c) = w\bm{s}_\phi(\bm{u},t,c)+(1-w)\bm{s}_\phi(\bm{u},t,\emptyset),
\end{equation}
where $\bm{s}_\phi(\bm{u},t,\emptyset) =\bm{s}_{\bm{\phi}}(\bm{u},t) \approx\nabla\log q_t(\bm{u}) $ denotes the unconditional score with $c$ set to the empty token.
Equation \eqref{eq:score-classifier} also holds for RDMs, and one can directly substitute the score function in equation \eqref{eq:reverse-sde} with $\tilde{\bm{s}}_\phi(\bm{u},t,c)$ to generate guided samples \citep{lou2023reflectedDM}.
\section{Proposed Method}
In this section, we present reflected flow matching (RFM), which learns a flow-based generative model for the data distribution $p_{\mathrm{data}}(\cdot)$ supported on a constrained domain $\Omega\subset \mathbb{R}^d$. The data domain $\Omega$ is assumed to be connected and compact with a non-empty interior $\Omega^o$ and a sufficiently regular boundary $\partial\Omega$.

\subsection{Reflected Ordinary Differential Equation}

To model flows over the constrained domain $\Omega$, we add an additional term for the boundary constraint to the vanilla ODE \eqref{eq:particle-ODE}, inspired by \citet{lou2023reflectedDM}.
Concretely, let $\bm{L}_t$ reflect the outward normal direction at $\partial\Omega$.
Given a initial point $\bm{x}$, our \emph{reflected ODE} for describing a time-dependent flow $\bm{x}_t=\bm{\phi}_t(\bm{x})$ takes the form
\begin{subequations}\label{eq:reflected-particle-ODE}
\arraycolsep=1.8pt
\def\arraystretch{1.5}
\begin{align}
\mathrm{d}\bm{\phi}_t(\bm{x}) &= \bm{v}_t(\bm{\phi}_t(\bm{x}))\mathrm{d}t + \dif\bm{L}_t,\\
\bm{\phi}_0(\bm{x}) &= \bm{x},
\end{align}
\end{subequations}
where $\bm{v}_t(\cdot): \Omega\to\mathbb{R}^d$ defines velocity field over $\Omega$ for each time $t\in[0,1]$. In equation \eqref{eq:reflected-particle-ODE}, the velocity term $\bm{v}_t(\bm{\phi}_t(\bm{x}))$ describes the motion of particles in $\Omega^o$ just as in the vanilla ODE \eqref{eq:particle-ODE}, and the reflection term $\dif\bm{L}_t$ compensates the outward velocity at $\partial\Omega$ by pushing the motion back to the domain $\Omega$.
That is, upon the motion's hitting the boundary $\partial \Omega$ at $\bm{x}_t$, $\dif\bm{L}_t$ neutralizes the outward normal-pointing component.
We give the following theorem for the existence and uniqueness of the solution to the reflected ODE \eqref{eq:reflected-particle-ODE}, which is a corollary of  \citet[Theorem 2.5.4]{pilipenko2014reflectedSDE} (see Appendix \ref{proof:existence} for details).
\begin{theorem}\label{thm:existence}
Assume i) the domain $\Omega$ satisfies the uniform exterior sphere condition and the uniform cone condition;
ii) the velocity field $\bm{v}_t(\bm{x})$ is Lipschitz continuous in $\bm{x}\in\Omega$ (uniformly on $t$).
Then the solution to the reflected ODE \eqref{eq:reflected-particle-ODE} exists and is unique on $t\in[0,1]$.
\end{theorem}

\begin{remark}
The reflected ODE \eqref{eq:reflected-particle-ODE} recovers the vanilla ODE \eqref{eq:particle-ODE} in flow matching by taking $\Omega=\mathbb{R}^d$.
In this case, the reflection term $\dif\bm{L}_t$ disappears because the motion will never hit the boundary.
\end{remark}

With the initial point $\bm{x}$ distributed as a simple prior distribution $p_0(\cdot)$ supported on $\Omega$, the PDF $p_t(\cdot)$ of the flow $\bm{\phi}_t(\bm{x})$ can also be used to construct a probability path which connects $p_0(\cdot)$ and a target distribution $p_1(\cdot)$ that closely approximates the data distribution $p_{\mathrm{data}}(\cdot)$.
Moreover, due to the reflection term $\dif\bm{L}_t$, the probability path $p_t(\bm{x})$ is naturally supported on $\Omega$
and evolves according to the following continuity equation with the Neumann boundary condition\footnote{Informally, when a particle hits the boundary, the outward-normal part of its velocity would be neutralized by the reflection term $\mathrm{d}L_t$, which would make the velocity satisfy the Neumann boundary condition.} \citep{schuss2015brownian}
\begin{subequations}\label{eq:continuity}
\arraycolsep=1.8pt
\def\arraystretch{1.5}
\begin{align}
\frac{\partial }{\partial t} p_t(\bm{x}) = -\nabla\cdot\left(\bm{v}_t(\bm{x}) p_t(\bm{x})\right),&\  \bm{x}\in\Omega^o,\label{eq:continuity-1} \\
p_t(\bm{x})\bm{v}_t(\bm{x})\cdot  \bm{n}(\bm{x}) = 0,&\ \bm{x}\in\partial\Omega,\label{eq:continuity-2}
\end{align}
\end{subequations}
where $\bm{n}(\bm{x})$ is the outward normal vector at $\bm{x}$ on $\partial\Omega$.

\subsection{Reflected Flow Matching}
Similarly to \citet{lipman2023FM}, the basic idea of RFM is to learn a parametrized velocity model $\bm{v}_{\theta}(\bm{x},t)$, which leads to a deep model of $\bm{\phi}_t$ called \emph{reflected CNF}, by minimizing the following RFM loss
\begin{equation}
\mathcal{L}_{\textrm{RFM}}(\bm{\theta})=\int_{0}^1 \mathbb{E}_{p_t(\bm{x})}\|\bm{v}_{\bm{\theta}}(\bm{x}, t)-\bm{v}_t(\bm{x})\|^2\mathrm{d}t,
\end{equation}
given a target probability path $p_t(\bm{x})$ and a corresponding velocity field $\bm{v}_t(\bm{x})$.
We can construct such a probability path and velocity field as follows.
Let $p_{t}(\bm{x}|\bm{x}_1)$ be a conditional probability path satisfying $p_0(\bm{x}|\bm{x}_1)=p_0(\bm{x}), p_1(\bm{x}|\bm{x}_1)\approx \delta(\bm{x}_1)$.
The marginal probability path then takes the form
\begin{equation}\label{eq:marginal-probability-path}
p_{t}(\bm{x}) = \int_\Omega p_{t}(\bm{x}|\bm{x}_1)p_{\mathrm{data}}(\bm{x}_1)\dif \bm{x}_1.
\end{equation}
Let $\bm{v}_t(\bm{x}|\bm{x}_1)$ be a conditional velocity field that generates the conditional probability path $ p_{t}(\bm{x}|\bm{x}_1)$ and defines the marginal velocity field
\begin{equation}\label{eq:marginal-velocity-field}
\bm{v}_t(\bm{x}) = \int_\Omega \bm{v}_t(\bm{x}|\bm{x}_1)\frac{p_t(\bm{x}|\bm{x}_1)p_{\mathrm{data}}(\bm{x}_1)}{p_t(\bm{x})}\dif \bm{x}_1.
\end{equation}
Note that we require $\mathrm{supp}\left(p_{t}(\bm{x}|\bm{x}_1)\right)\subset \Omega$ for all $t$ and $\bm{x}_1$, in contrast to the vanilla FM \citep{lipman2023FM}.
In fact, the velocity field $\bm{v}_t(\bm{x})$ in equation \eqref{eq:marginal-velocity-field} exactly generates the probability path $p_{t}(\bm{x})$ in equation \eqref{eq:marginal-probability-path}, which is formalized in Theorem \ref{thm:probability-path}.

\begin{theorem}\label{thm:probability-path}
Assume $\mathrm{supp}\left(p_{t}(\bm{x}|\bm{x}_1)\right)\subset \Omega$ for all $t$ and $\bm{x}_1$.
For any target distribution $p_1(\bm{x}_1)$, if the conditional probability path $p_t(\bm{x}|\bm{x}_1)$ and conditional velocity field $\bm{v}_t(\bm{x}|\bm{x}_1)$ satisfy the continuity equation \eqref{eq:continuity}, then marginal probability path $p_t(\bm{x})$ and the marginal velocity field $\bm{v}_t(\bm{x})$ also satisfy the continuity equation \eqref{eq:continuity}.
\end{theorem}
\begin{proof}
The fact that $p_t(\bm{x})$ and $\bm{v}_t(\bm{x})$ satisfy equation \eqref{eq:continuity-1} is from \citet[Theorem 1]{lipman2023FM}.
To prove $p_t(\bm{x})$ and $\bm{v}_t(\bm{x})$ satisfy the Neumann boundary condition \eqref{eq:continuity-2}, it suffices to combine
\[
p_t(\bm{x})\bm{v}_t(\bm{x}) = \int_\Omega  \bm{v}_t(\bm{x}|\bm{x}_1)p_t(\bm{x}|\bm{x}_1)p_{\mathrm{data}}(\bm{x}_1)\dif \bm{x}_1
\]
and $p_t(\bm{x}|\bm{x}_1)\bm{v}_t(\bm{x}|\bm{x}_1)\cdot \bm{n}(\bm{x})=0$ for $\bm{x}\in\partial\Omega$.
\end{proof}
Similarly to the CFM objective in equation \eqref{eq:cfm}, the velocity model $\bm{v}_{\bm{\theta}}(\bm{x},t)$ in reflected CNFs can be trained by optimizing the following conditional reflected flow matching (CRFM) objective
\begin{equation}\label{eq:crfm}
\mathcal{L}_{\textrm{CRFM}}(\bm{\theta})=\int_{0}^1 \mathbb{E}_{ p_t(\bm{x}|\bm{x}_1)p_{\mathrm{data}}(\bm{x}_1)}\|\bm{v}_{\bm{\theta}}(\bm{x}, t)-\bm{v}_t(\bm{x}|\bm{x}_1)\|^2\mathrm{d}t,
\end{equation}
as proved in Theorem \ref{thm:crfm} (see Appendix \ref{proof:crfm} for proof).
\begin{theorem}\label{thm:crfm}
Assume that $p_t(\bm{x})>0$ for all $\bm{x}\in\Omega^o$ and $t\in[0,1]$. Then $\nabla \mathcal{L}_{\textrm{RFM}}(\bm{\theta}) = \nabla \mathcal{L}_{\textrm{CRFM}}(\bm{\theta})$.
\end{theorem}

Moreover, we prove the following Wasserstein bound for the reflected CNFs trained with RFM (see Appendix \ref{proof:wasserstein} for proof), similar to  \citet[Proposition 3]{albergo2022stochasticinterpolant}.
\begin{theorem}[Wasserstein Bound]\label{thm:wasserstein}
Assume $\Omega$ is convex and the velocity model $\bm{v}_{\bm{\theta}}(\bm{x},t)$ is $M$-Lipschitz in $\bm{x}$ (uniformly on $t$). 
Let $p_{\bm{\theta},t}(\bm{x})$ be the probability path of the reflected CNF induced by $\bm{v}_{\bm{\theta}}(\bm{x},t)$ starting from the same prior distribution $p_0(\bm{x})$.
Then the squared Wasserstein-2 distance between $p_{1}(\bm{x})$ and $p_{\bm{\theta},1}(\bm{x})$ is bounded by
\begin{equation}
W_2^2\left(p_{1}(\bm{x}),p_{\bm{\theta},1}(\bm{x})\right)\leq e^{1+2M}\mathcal{L}_{\textrm{RFM}}(\bm{\theta}).
\end{equation}
\end{theorem}

\subsection{Construction of Conditional Velocity Fields}

\begin{figure}
    \centering
    \includegraphics[width=\linewidth]{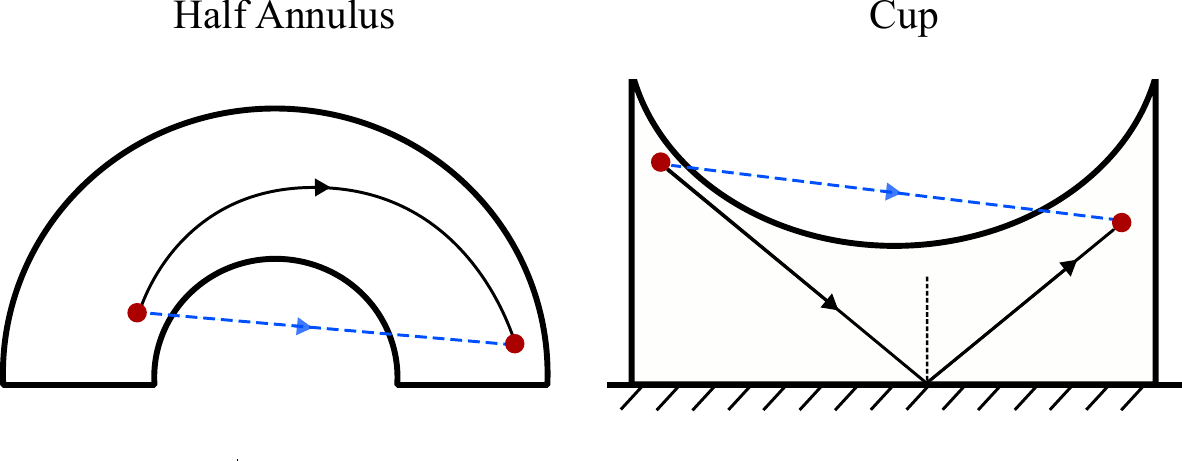}
    \caption{Illustration of the conditional velocity fields on two nonconvex domains: half annulus (left) and cup (right).
    The black solid curve is the designed conditional velocity field within the domain.
    The OT conditional velocity field is represented by the blue dashed segment which violates the domain constraint.
    }
    \label{fig:rfm-illustration}
\end{figure}
One remaining thing is the choice of the conditional probability path $p_t(\bm{x}|\bm{x}_1)$ and the conditional velocity field $\bm{v}_t(\bm{x}|\bm{x}_1)$ on $\Omega$.
Based on reflected ODEs, both of them can be simply derived from a conditional flow $\bm{\phi}_t(\bm{x}|\bm{x}_1)$, which satisfies that $\bm{\phi}_0(\bm{x}|\bm{x}_1)=\bm{x}$, $\bm{\phi}_1(\bm{x}|\bm{x}_1)\approx\bm{x}_1$, and $\bm{\phi}_t(\bm{x}|\bm{x}_1)\in \Omega$ for all $\bm{x},\bm{x}_1\in \Omega$ and $t\in[0,1]$.
In fact, with a data sample $\bm{x}_1$ from $p_{\mathrm{data}}(\cdot)$, 
\begin{subequations}
\arraycolsep=1.8pt
\def\arraystretch{1.5}
\begin{align}
\bm{v}_t(\bm{x}|\bm{x}_1) &=\bm{\phi}'_t(\bm{\phi}^{-1}_t(\bm{x}|\bm{x}_1)|\bm{x}_1),\\
p_t(\bm{x}|\bm{x}_1)&=p_0(\bm{\phi}^{-1}_t(\bm{x}|\bm{x}_1))\left|\mathrm{det}\left(\nabla_{\bm{x}}\bm{\phi}^{-1}_t(\bm{x}|\bm{x}_1)\right)\right|,
\end{align}
\end{subequations}
naturally satisfy the continuity equation \eqref{eq:continuity}.
Here we provide two concrete examples.

\begin{example}[Convex Domain]\label{example:convex}
Let $\Omega$ be a convex domain. Using the same idea of OT conditional velocity field, we set
\begin{equation}\label{eq:OT-probability-paths}
\bm{\phi}_t(\bm{x}|\bm{x}_1) = (1-(1-\sigma_{\min})t)\bm{x} + (1-\sigma_{\min})t\bm{x}_1,
\end{equation}
where $\sigma_{\min}$ is a sufficiently small number such that the resulting $p_1(\cdot|\bm{x}_1)$ is concentrated around $\bm{x}_1$.
Note that equation \eqref{eq:OT-probability-paths} is different from the vanilla OT conditional velocity field in FM in that it has a coefficient $(1-\sigma_{\min})$ on $\bm{x}_1$ to ensure that $\bm{\phi}_t(\bm{x}|\bm{x}_1)\in\Omega$ due to the convexity of $\Omega$.
A typical example of convex domains is the convex polytope defined as
\begin{equation}\label{eq:convex-polytope}
\Omega = \left\{\bm{u}\,|\, \bm{A}\bm{u}< \bm{b}\right\}\subset \mathbb{R}^d,
\end{equation}
where $\bm{A}\in \mathbb{R}^{m\times d}$, $\bm{b}\in \mathbb{R}^m$, and $m$ is the number of linear constraints.
By taking $\bm{A} = [\bm{I}_d,-\bm{I}_d]^\prime$ and $\bm{b}=\bm{1}_{2d}$ where $\bm{I}_d$ is the $d$-dimensional identity matrix and $\bm{1}_{2d}$ is an all-ones vector of length $2d$, the corresponding constrained domain becomes a hypercube $\Omega=[-1,1]^d$, which is frequently encountered in image generation tasks.
\end{example}

\begin{example}[Half Annulus]\label{example:ring}
Consider a two-dimensional half annulus area
\begin{equation}
    \Omega = \{(x_1,x_2)\, |\, r^2\leq x_1^2+x_2^2\leq R^2,x_2\geq 0\}\subset \mathbb{R}^2
\end{equation}
with $0<r<R$. 
Let $\bm{x}=(\|\bm{x}\|\cos\alpha,\|\bm{x}\|\sin\alpha)'$ and $\bm{x}_1=(\|\bm{x}_1\|\cos\alpha_1,\|\bm{x}_1\|\sin\alpha_1)'$ where $\alpha,\alpha_1 \in [0,\pi]$.
The conditional flow that (approximately) connects $\bm{x}$ and $\bm{x}_1$ can be given by $\bm{\phi}_t(\bm{x}|\bm{x}_1)=(x_t\cos\alpha_t,x_t\sin\alpha_t)'$ where
\begin{subequations}
\label{eq:ring-conditional-probability-path}
\begin{align}
x_t &= (1-(1-\sigma_{\min})t)\|\bm{x}\| + (1-\sigma_{\min})t\|\bm{x}_1\|,\\
\alpha_t& = (1-(1-\sigma_{\min})t)\alpha + (1-\sigma_{\min})t\alpha_1.
\end{align}
\end{subequations}
It is easy to verify that $\bm{\phi}_t(\bm{x}|\bm{x}_1)\in\Omega$ for all $\bm{x}$, $\bm{x}_1$, and $t$.
Intuitively, equation \eqref{eq:ring-conditional-probability-path} forms an OT conditional velocity field that (approximately) pushes $\bm{x}$ to $\bm{x}_1$ in the polar coordinates.
See the left plot in  Figure \ref{fig:rfm-illustration} for an illustration.
\end{example}

Compared to RDMs \citep{lou2023reflectedDM}, the superiority of RFM can be summarized in three key aspects: i) one can efficiently sample from $p_t(\bm{x}|\bm{x}_1)$ through the analytical conditional flow $\bm{\phi}_{t}(\bm{x}|\bm{x}_1)$ to obtain the training data in $\mathcal{L}_{\textrm{CRFM}}(\bm{\theta})$ without requiring additional geometric techniques;
ii) the conditional velocity field $\bm{v}_{t}(\bm{x}|\bm{x}_1)$ has an analytical form, eliminating the need for potentially complicated approximations;
iii) in principle, the prior distribution can take any distribution on arbitrary domain $\Omega$, thereby enhancing the flexibility of reflected CNFs.

\begin{algorithm}[t]
\caption{Sampling from Reflected CNFs}
\label{alg:sampling}
\begin{algorithmic}
\STATE {\bfseries Input:} A velocity model $\bm{v}_\theta(\bm{x},t)$; the prior distribution $p_0(\cdot)$; total number of discretization steps $K$; an increasing time sequences $\{\gamma_k\}_{k=0}^K$ such that $\gamma_0=0$ and $\gamma_K=1$; an one-step ODE solver  $\mathrm{SOLVER}(\cdot)$; a reflection function $\mathrm{Refl}_{\partial\Omega}(\cdot)$.
\STATE {\bfseries Output:} A sample $\bm{x}_1$ from the reflected CNF.
\STATE Sample a point $\bm{x}_0$ from $p_0(\cdot)$;
\FOR{$k=1$ \textbf{to} $K$}
\STATE $\bar{\bm{x}}_{\gamma_k} \leftarrow \mathrm{SOLVER}\left(\bm{x}_{\gamma_{k-1}}, \gamma_{k-1},\gamma_k,\bm{v}_\theta\right)$;
\STATE $(\Delta, \bm{\alpha}, \bm{y}) \leftarrow\left(\|\bar{\bm{x}}_{\gamma_k}-\bm{x}_{\gamma_{k-1}}\|, \frac{\bar{\bm{x}}_{\gamma_k}-\bm{x}_{\gamma_{k-1}}}{\|\bar{\bm{x}}_{\gamma_k}-\bm{x}_{\gamma_{k-1}}\|}, \bm{x}_{\gamma_{k-1}}\right)$;
\WHILE{TRUE}
\STATE Calculate the first intersection $\bm{y}^\prime$ of the ray $\bm{y}+\bm{\alpha}s (s\geq 0)$ and the boundary $\partial \Omega$;
\STATE $\Delta'\leftarrow\|\bm{y}'-\bm{y}\|$;
\IF{$\Delta\leq \Delta'$}
\STATE Stop the inner loop;
\ENDIF
\STATE $(\Delta, \bm{\alpha}, \bm{y}) \leftarrow 
\left(\Delta-\Delta', \mathrm{Refl}_{\partial\Omega}(\bm{\alpha},\bm{y}'), \bm{y}'\right)$;
\ENDWHILE
\STATE $\bm{x}_{\gamma_{k}} \leftarrow \bm{y}+\bm{\alpha}\Delta \in \Omega$;
\ENDFOR
\end{algorithmic}
\end{algorithm}

\subsection{Sampling from Reflected CNFs}\label{sec:sampling}
To generate a sample from the learned reflected CNFs, we start from the prior distribution and simulate the reflected ODE \eqref{eq:reflected-particle-ODE} where $\bm{v}_t(\bm{x})$ is replaced by $\bm{v}_{\theta}(\bm{x},t)$.
Assume an increasing sequence $0= \gamma_0<\cdots<\gamma_K=1$ of the interpolation time points for ODE simulation.
Here, we only discuss the simulation methods with fixed step sizes, and those with adaptive step sizes can be derived similarly.

First of all, we sample a point $\bm{x}_0$ from the prior distribution $p_0(\cdot)$ as our initialization as $t=0$.
In the $k$-th step, given the current position $\bm{x}_{\gamma_{k-1}}$, we need to calculate the next position $\bm{x}_{\gamma_{k}}$.
For ODEs without spatial constraints, this one-step update can be achieved by many widely-used algorithms, e.g., Euler method, Runge-Kutta method, etc. These algorithms can generally be represented by
\begin{equation}
\bar{\bm{x}}_{\gamma_k} = \mathrm{SOLVER}\left(\bm{x}_{\gamma_{k-1}}, \gamma_{k-1},\gamma_k,\bm{v}_\theta\right),
\end{equation}
where the ODE solver $\mathrm{SOLVER}(\cdot)$ takes the current position $\bm{x}_{\gamma_{k-1}}$, the starting and ending time $\gamma_{k-1},\gamma_k$, and a velocity field $\bm{v}_\theta$ as inputs.
However, if the ODE solver does not account for the boundary constraints, the next position $\bar{\bm{x}}_{\gamma_k}$ may extend beyond the constrained domain $\Omega$.

Inspired by the geometry implication of the reflection term $\mathrm{d}\bm{L}_t$, we handle the case $\bar{\bm{x}}_{\gamma_k}\notin\Omega$ by applying reflections iteratively to the segment from $\bm{x}_{\gamma_{k-1}}$ to $\bar{\bm{x}}_{\gamma_k}$. 
Concretely, let $\Delta=\|\bar{\bm{x}}_{\gamma_k}-\bm{x}_{\gamma_{k-1}}\|$ and $\bm{\alpha}= \frac{\bar{\bm{x}}_{\gamma_k}-\bm{x}_{\gamma_{k-1}}}{\|\bar{\bm{x}}_{\gamma_k}-\bm{x}_{\gamma_{k-1}}\|}\in\mathbb{R}^d$ be the length and direction of the one-step update, and $\bm{y}=\bm{x}_{\gamma_{k-1}}$ is the current position. 
We first calculate the first intersection $\bm{y}^\prime$ of the ray $\bm{y}+\bm{\alpha}s (s\geq 0)$ and the boundary $\partial \Omega$.
If $\|\bm{y}'-\bm{y}\|=:\Delta'<\Delta$, we continue the iteration and update the triplet $(\Delta, \bm{\alpha}, \bm{y})$ by
\begin{equation}\label{eq:triple-update}
(\Delta, \bm{\alpha}, \bm{y}) \leftarrow 
\left(\Delta-\Delta', \mathrm{Refl}_{\partial\Omega}(\bm{\alpha},\bm{y}'), \bm{y}'\right),
\end{equation}
where $\mathrm{Refl}_{\partial\Omega}(\bm{\alpha},\bm{y}')$ gives the reflected velocity of $\bm{\alpha}$ at $\bm{y}'$ on $\partial\Omega$; otherwise, we stop the iteration and calculate
\begin{equation}
\bm{x}_{\gamma_{k}} = \bm{y}+\bm{\alpha}\Delta \in \Omega
\end{equation}
as the next position.
The whole procedure for sampling from reflected CNFs is summarized in Algorithm \ref{alg:sampling}.

Note that for certain special domains $\Omega$, it is possible to directly compute the next position instead of applying reflections iteratively.
For example, for $\Omega=[-1,1]^d$, we can directly compute $\bm{x}_{\gamma_k}=1-|(\bar{\bm{x}}_{\gamma_k}+1)\,\mathrm{mod}\, 4-2|$.
This observation is crucial to efficient sampling in image generation tasks.

\paragraph{Flow Guidance}
Inspired by the classifier-free guidance for score-based conditional generation (Section \ref{sec:guidance-diffusion}), we consider a similar scheme for guided generation based on the reflected ODE \eqref{eq:reflected-particle-ODE}.
With a $c$-conditioned velocity field $\bm{v}_t(\bm{x}|c)$ which pushes the prior distribution $p_0(\bm{x}|c)=p_0(\bm{x})$ to $p_1(\bm{x}|c)$, we defines the \textit{guided velocity field} as
\begin{equation}\label{eq:guided-velocity-field}
\tilde{\bm{v}}_{t}(\bm{x}|c) = w\bm{v}_t(\bm{x}|c) + (1-w)\bm{v}_t(\bm{x}|\emptyset),
\end{equation}
where $w$ is the guidance weight and $\bm{v}_t(\bm{x}|\emptyset)$ is the unconditioned velocity with $c$ set to the empty token.
A similar idea is also considered in \citet{dao2023latentFM, Zheng2023GuidedFF}.
According to the reflected ODE, the \textit{guided probability path} $\tilde{p}_t(\bm{x}|c)$ is then defined implicitly by the solution to the following continuity equation with the Neumann boundary condition
\begin{subequations}
\arraycolsep=1.8pt
\def\arraystretch{1.5}
\begin{align}
\frac{\partial }{\partial t} \tilde{p}_t(\bm{x}|c) = -\nabla_{\bm{x}}\cdot\left(\tilde{\bm{v}}_{t}(\bm{x}|c) \tilde{p}_t(\bm{x}|c)\right),&\  \bm{x}\in\Omega^o,\\
\tilde{p}_t(\bm{x}|c)\tilde{\bm{v}}_{t}(\bm{x}|c)\cdot  \bm{n}(\bm{x}) = 0,&\ \bm{x}\in\partial\Omega,
\end{align}
\end{subequations}
where $\bm{n}(\bm{x})$ is the outward normal direction at $\bm{x}\in\partial\Omega$.
In particular, one has $\tilde{p}_t(\bm{x}|c)=p_t(\bm{x})$ or $\tilde{p}_t(\bm{x}|c)=p_t(\bm{x}|c)$ if $w=0$ or $w=1$ respectively.

The following theorem implies that in the vanilla FM setting, the velocity field is connected to that in the ODE counterpart of the reversed SDE (i.e., $\lambda=0$ in equation \eqref{eq:reverse-sde}), establishing an equivalence between equation \eqref{eq:score-classifier} and \eqref{eq:guided-velocity-field}.
\begin{theorem}\label{thm:guided-velocity-field}
Let $\Omega=\mathbb{R}^d$ and the prior distribution $p_0(\bm{x})$ be a standard Gaussian distribution. Assume the OT conditional velocity field for FM, i.e., $\bm{\phi}_{t}(\bm{x}|\bm{x}_1)=t\bm{x}_1+(1-(1-\sigma_{\min})t)\bm{x}$.
Then
\begin{equation}
\bm{v}_t(\bm{x})=\frac{1}{t}\bm{x}+\frac{1-(1-\sigma_{\min})t}{t}\nabla\log p_{t}(\bm{x}).
\end{equation}
\end{theorem}
The proof of Theorem \ref{thm:guided-velocity-field} can be found in Appendix \ref{proof:guided-velocity-field}.
\begin{figure*}[t]
    \centering
    \includegraphics[width=\linewidth]{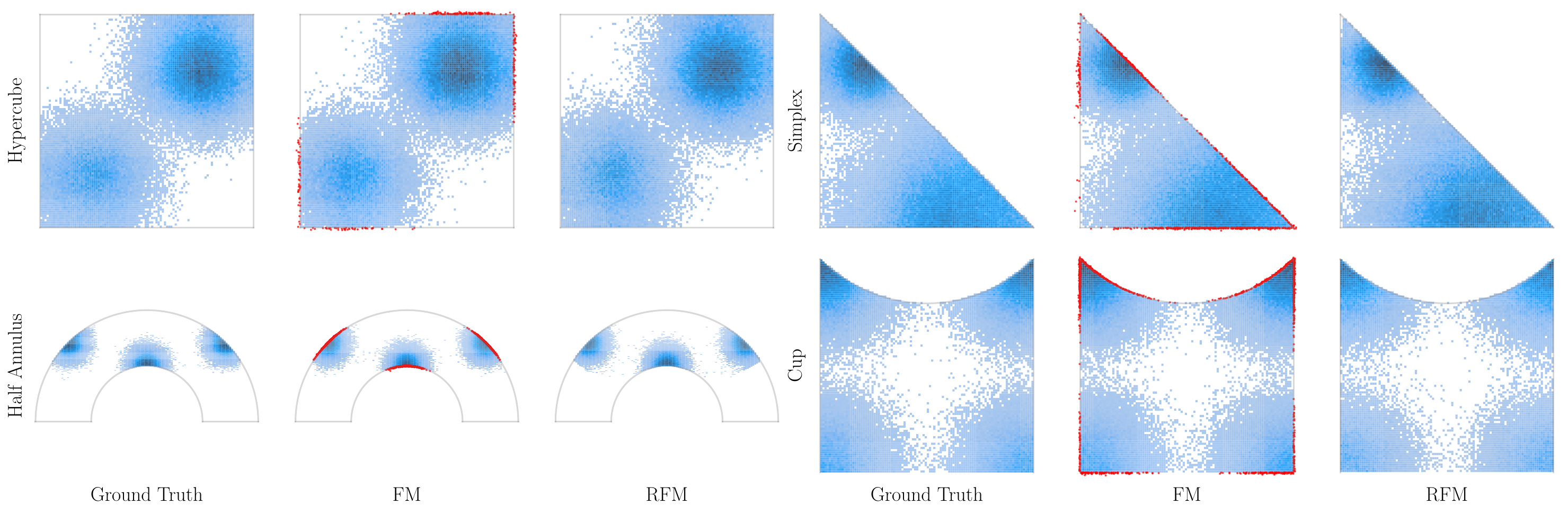}
    \caption{The histplots of samples obtained by different methods compared to the ground truth on the two-dimensional hypercube, simplex, and cup data set.
    Samples out of the constrained domain are plotted with red dots.
    The total sample size is 100,000.
    }
    \label{fig:histplot2D-heun3}
\end{figure*}

\begin{table*}[t]
\centering 
\caption{KL divergences to the ground truth obtained by different methods on low-dimensional generation tasks.
For each method, we calculate the KL divergence using the Python ITE module \citep{szabo2014ITE} with a sample size of 50,000.
The results are averaged over 10 independent runs with standard deviation in the brackets.
}
\label{tab:KL2D-heun3}
\vspace{0.5em}
\begin{tabular}{lcccccc}
\toprule
\multirow{2}{*}{Method}&\multicolumn{2}{c}{Hypercube}&\multicolumn{2}{c}{Simplex}&Half Annulus&Cup \\ 
\cmidrule(l){2-3}\cmidrule(l){4-5}\cmidrule(l){6-6}\cmidrule(l){7-7}
&$d=2$&$d=10$&$d=2$&$d=10$&$d=2$&$d=2$\\
\midrule
RDM&0.0110(0.0003)&0.4798(0.0006)&0.0055(0.0001)&0.7613(0.0036)& N/A &N/A\\
FM&0.0044(0.0010)&\textbf{0.0224(0.0010)}&0.0062(0.0011)&0.0522(0.0022)&0.0083(0.0008)&0.0122(0.0012)\\
FM$^\ast$&\textbf{0.0021(0.0010)}&0.0257(0.0020)&0.0045(0.0009)&0.0460(0.0022)&0.0085(0.0011)&0.0118(0.0020)\\
RFM&\textbf{0.0021(0.0011)}&0.0243(0.0011)&\textbf{0.0027(0.0011)}&\textbf{0.0452(0.0023)}&\textbf{0.0038(0.0009)}&\textbf{0.0058(0.0010)}\\
\bottomrule
\end{tabular}
\end{table*}

\begin{table}[t]
\centering 
\setlength\tabcolsep{3pt}
\caption{Constraint violation ratio (\textperthousand) of different methods on low-dimensional generation tasks.
We collect 500,000 samples to calculate the constraint violation ratio.
}
\label{tab:outier2D-heun3}
\vspace{0.5em}
\resizebox{\linewidth}{!}{
\begin{tabular}{lcccccc}
\toprule
\multirow{2}{*}{Method}&\multicolumn{2}{c}{Hypercube}&\multicolumn{2}{c}{Simplex}&Half Annulus&Cup \\ 
\cmidrule(l){2-3}\cmidrule(l){4-5}\cmidrule(l){6-6}\cmidrule(l){7-7}
&$d=2$&$d=10$&$d=2$&$d=10$&$d=2$&$d=2$\\
\midrule
FM&3.1&26.9&9.0&86.2&10.4&13.9\\
FM$^\ast$&0.2&8.6&5.9&33.7&10.2&10.2\\
RFM&0.0&0.0&0.0&0.0&0.0&0.0\\
\bottomrule
\end{tabular}
}
\end{table}

\section{Experiments}

\subsection{Low-dimensional Toy Examples}

We first test the effectiveness of RFM for modeling probability distributions on low-dimensional constrained domains, including hypercube, simplex, half annulus, and cup.

\paragraph{Hypercube and Simplex} 
We consider the $d$-dimensional hypercube $\Omega=[-1,1]^d$ and the $d$-dimensional simplex
\[
\Omega=\left\{(x_1,\ldots,x_d)'\left|x_1\geq 0,\ldots,x_d\geq 0,\sum_{i=1}^dx_i\leq 1\right.\right\}.
\]
Both domains are convex and satisfy the condition in Example \ref{example:convex}.
We consider two cases here: $d=2$ and $d=10$.
For RFM, we use the OT conditional velocity field in equation \eqref{eq:OT-probability-paths} with $\sigma_{\min}=10^{-5}$.

\paragraph{Half Annulus} 
We consider a two-dimensional half annulus area which is defined in Example \ref{example:ring} and choose the conditional velocity field in equation \eqref{eq:ring-conditional-probability-path} with $\sigma_{\mathrm{min}}=10^{-5}$ for RFM.

\paragraph{Cup}
We also consider a non-convex two-dimensional cup area defined as 
\[
\Omega=\{(x_1,x_2)'| -1\leq x_1\leq 1, x_2\geq 0, x_1^2+(x_2-3)^2\leq 2\}.
\]
Let $\bm{x}=(x_1,x_2)$ and $\bm{x}_1=(x_{1,1}, x_{1,2})$ be the prior sample and data sample.
We define the conditional flow as
\[
\bm{\phi}_t(\bm{x}|\bm{x}_1) = \left((1-t)x_1+tx_{1,1},|(1-t)x_2-tx_{1,2}|\right)',
\]
which is a broken line connecting $\bm{x}$ and $\bm{x}_1$ with one reflection at the $x$-axis (Figure \ref{fig:rfm-illustration}, right plot).

For all these domains above, we set the prior distribution $p_{0}(\cdot)$ to the standard Gaussian distribution for CNFs (trained with FM) and the uniform distribution over the constrained domain $\Omega$ for reflected CNFs (trained with RFM).
As an additional baseline, the CNFs (trained with FM) where $p_0(\cdot)$ is a uniform distribution are also considered and denoted by FM$^\ast$.
We choose the target distribution $p_{\mathrm{data}}(\cdot)$ as the mixture of Gaussians which is truncated by the domain $\Omega$ \citep{fishman2023diffusionCD}.
On all domains, the OT conditional velocity field with $\sigma_{\mathrm{min}}=10^{-5}$ is chosen for FM.
The results are collected after 200,000 iterations with a batch size of 512.
We use the popular fixed-step third-order Heun algorithm \citep{chen2018neuralode} and set the number of function evaluations (NFE) as 300 to sample from the CNFs and the reflected CNFs.
We also consider the RDM baseline implemented by \citet{fishman2023diffusionCD} and generate the samples using the Euler-Maruyama algorithm \citep{platen1992em} with 300 NFE.
See more details of implementation in Appendix \ref{app:implimentation-details-toy}.

Table \ref{tab:KL2D-heun3} reports the approximation accuracies measured by the Kullback–Leibler (KL) divergences to the ground truth obtained by different methods.
We find that in all cases, RFM performs on par or better than FM and FM$^\ast$ and consistently outperforms RDM.
Interestingly, the results of RDM get worse in the 10-dimensional cases, partially due to the requirement for more discretization steps.
One can also see from Table \ref{tab:outier2D-heun3} that a large proportion of samples violating the constraints are generated by FM and FM$^\ast$, while our RFM enjoys zero constraint violation ratio by design.
This demonstrates the effectiveness of the reflection operation in CNFs.
Interestingly, the boundary violation ratio of CNFs can be somewhat reduced by employing a uniform prior distribution, as evidenced by the FM$^\ast$ results.

Figure \ref{fig:histplot2D-heun3} shows the histplots of samples obtained by different methods.
We see that both CNFs and reflected CNFs (trained by FM and RFM respectively) can generate samples similar to the data distributions.
Moreover, samples from reflected CNFs all stay within the domains while those from CNFs can violate the constraints (the red dots).
It is worth noting that the conditional velocity fields of RFM on the nonconvex half annulus and cup areas work well, although they are not optimal transports.
More results about alternative ODE solvers and learned velocity fields can be found in Appendix \ref{app:exp-toy-example}.

\begin{table}[t]
\centering 
\caption{Sample quality measured by FID score ($\downarrow$) obtained by different methods on CIFAR-10 ($32\times 32$). 
Methods for constrained generation are marked in the shaded area.
The result with $\ast$ is produced by ourselves using the official checkpoint in \citet{lou2023reflectedDM} and the predictor-corrector (PC) algorithm.
The other results are from their original papers.
}
\vspace{0.5em}
\label{tab:cifar-results}
\resizebox{1\linewidth}{!}{
\begin{tabular}{lcc}
\toprule
Method                                     &FID           &NFE \\
\midrule
NCSN++ \citep{song2020score}               &2.20          &2000 \\
DDPM++ \citep{song2020score}               &2.41          &2000 \\
Subspace NCSN++ \citep{jing2022subspace}   &2.17          &2000 \\
DDIM \citep{song2020denoising}             &4.16          &100 \\
FM \citep{lipman2023FM}                    &6.35          &142 \\ 
\midrule
\rowcolor{gray!20}RDM \citep{lou2023reflectedDM}             &2.72          &2000\\
\rowcolor{gray!20}RDM \citep{lou2023reflectedDM}             &44.40$^\ast$  &200\\
\rowcolor{gray!20}\textbf{RFM (ours)}                        &4.76          &139 \\ 
\bottomrule
\end{tabular}
}
\vspace{-1em}
\end{table}
\begin{figure*}[t]
    \centering
    \includegraphics[width=\linewidth]{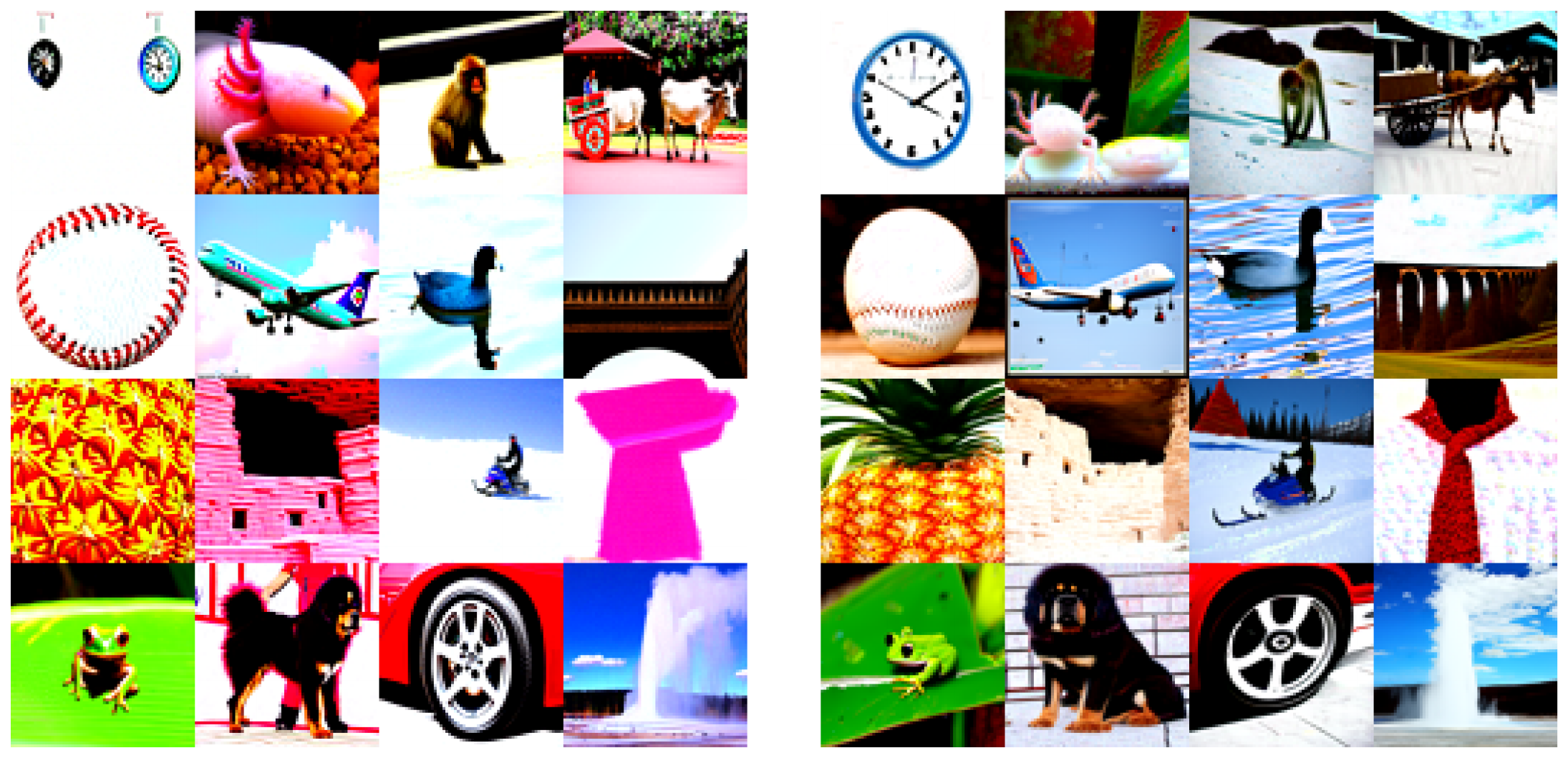}
    \caption{Class-conditioned guided samples from the CNFs trained with FM (left, $\textrm{NFE}=769$) and the reflected CNFs trained with RFM (right, $\textrm{NFE}=723$) with a high guidance weight $w=15$ on ImageNet ($64\times 64$).
    The samples from CNFs are clipped to $[0,255]$ after the ODE simulation.
    }
    \label{fig:imagenet64}
\end{figure*}
\subsection{Image Generation Benchmarks}
We then explore the performances of RFM for the unconditional image generation task on CIFAR-10 $(32\times 32)$ and the conditional generation task on ImageNet $(64\times 64)$.
For both data sets, the data are rescaled to a hypercube $\Omega=[-1,1]^{3\times m\times m}$ ($m$ is the image size), which is a convex domain (Example \ref{example:convex}).
For CNFs trained with FM, we set the prior distribution to be a standard Gaussian distribution and use the OT conditional velocity field.
For reflected CNFs trained with RFM, we set the prior distribution to be a truncated standard Gaussian distribution over $\Omega$ and use the OT conditional velocity field in equation \eqref{eq:OT-probability-paths}.
The architectures of the velocity model $\bm{v}_{\bm{\theta}}(\bm{x},t)$ for CIFAR-10 or the class-conditioned velocity model $\bm{v}_{\bm{\theta}}(\bm{x},t,c)$ for ImageNet employed by RFM and FM are the same as the UNet \citep{ronneberger2015unet} of the score function model in \citet{dhariwal2021diffusion}.
On CIFAR-10, the velocity model of RFM is optimized with Adam \citep{ADAM} and a constant learning rate of 0.0002 after a warm-up phase of 5000 training steps; on ImageNet, the velocity model of RFM is optimized with AdamW \citep{loshchilov2018adamw} and a constant learning rate of 0.0001 after a warm-up phase of 5000 training steps.
The total number of training steps is 800,000 on CIFAR-10 and 540,000 on ImageNet.
The batch size is set to 128 on CIFAR-10 and 2048 on ImageNet.
We also train CNFs with FM on ImageNet by ourselves using the same setting as RFM.
We use the adaptive-step Dormand–Prince method \citep{dormand1980dopri} with absolute and relative tolerances of $10^{-5}$ to sample from the CNFs and the reflected CNFs.
See more implementation details in Appendix \ref{app:cifar10} and \ref{app:imagenet64}.

\paragraph{Unconditional Image Generation}
In Table \ref{tab:cifar-results}, we draw 50,000 samples and report the Fréchet inception distance (FID) score \citep{Seitzer2020FID} as well as the NFE obtained by different methods.
Among the methods for constrained generation, we see that although RDM with the predictor-corrector (PC) sampler works well under a large NFE, it has a large performance drop when the NFE is small.
In contrast, RFM can obtain high-quality samples on constrained domains with small NFEs, inheriting the benefit of straighter velocity fields of FM.
Besides, we want to point out that the FID score produced by FM can be sensitive to the experimental setting (e.g., FID=3.66 reproduced by \citet{tong2023CFM}).
In comparison with FM, the benefit of RFM mainly comes from a zero boundary violation rate and more realistic image samples (see the next paragraph), and the FID score is not very sensitive to the boundary violation issue.
Image samples generated from reflected CNFs trained with RFM can be found in Figure \ref{fig:cifar10} of Appendix \ref{app:cifar-results}.

\begin{table}[t]
\centering 
\caption{Digit-level constraint violation ratio of FM and RFM on the ImageNet ($64\times 64$) conditional generation task.
(Reflected) CNFs are simulated by the fixed-step third-order Heun algorithm with 600 NFE.
}
\vspace{0.5em}
\label{tab:imagenet-violation-results}
\resizebox{\linewidth}{!}{
\begin{tabular}{lccc}
\toprule
Flow guidance weight $w$ & 7.5 & 15 & 25 \\
\midrule
FM & 51.28\%&	61.12\%&	64.51\%\\
RFM & 0.0\%&0.0\%&0.0\%\\
\bottomrule
\end{tabular}
}
\vspace{-1em}
\end{table}

\paragraph{Conditional Image Generation}
Figure \ref{fig:imagenet64} shows the class-conditioned samples generated by the CNFs trained with FM and the reflected CNFs trained with RFM under a high guidance weight using the flow guidance introduced in Section \ref{sec:sampling}.
We see that the vanilla CNFs tend to generate oversaturated and unnatural images due to the boundary violation.
In contrast, reflected CNFs trained with RFM effectively improve the sample quality by producing more faithful images with a comparable NFE.
Table \ref{tab:imagenet-violation-results} reports the constraint violation ratios of different methods with varying guidance weights.
We see that the constraint violation ratio of FM is large for this high-dimensional task and can increase as the flow guidance weight gets larger.

\section{Conclusion}
In this work, we presented reflected flow matching (RFM), an extension of flow matching for constrained domains through reflected ODEs/CNFs.
By leveraging conditional probability paths that adhere to the constraints, we show that the velocity fields in the reflected ODEs/CNFs can be effectively trained by matching the analytical conditional velocity fields in a simulation-free manner, akin to FM.
We demonstrate the efficacy of RFM across various choices of conditional velocity fields on low-dimensional convex and nonconvex domains.
On standard image benchmarks, we show that RFM performs on par or better than existing baselines with a small NFE and produces more faithful samples under high guidance weight.
Further applications of RFM to constrained generative modeling of high-resolution images, videos, cloud points, human motions, etc., will be interesting future directions.

\paragraph{Limitations}
To apply RFM, we have the following requirements for the constrained domain: i) the domain has a regular shape so that we can design a conditional flow that is confined to it. All convex domains satisfy this requirement. ii) the boundary has an analytical form to compute the intersection point and the reflection direction if necessary. iii) the conditional flow is easy to simulate and differentiate for velocity computation.
We leave exploring the design of conditional probability paths on more general domains to future work. 

\section*{Impact Statement}
This paper presents work whose goal is to advance the field of machine learning. There are many potential societal consequences of our work, none of which we feel must be specifically highlighted here.

\section*{Acknowledgements}
This work was supported by National Natural Science Foundation of China (grant no. 12201014 and grant no. 12292983).
The research of Cheng Zhang was support in part by National Engineering Laboratory for Big Data Analysis and Applications, the Key Laboratory of Mathematics and Its Applications (LMAM) and the Key Laboratory of Mathematical Economics and Quantitative Finance (LMEQF) of Peking University.
The authors appreciate the anonymous ICML reviewers for their constructive feedback.

\bibliography{main}
\bibliographystyle{icml2024}

\newpage
\appendix
\onecolumn

\section{Theoretical Results}\label{app:theoretical-results}
\subsection{Details of Theorem \ref{thm:existence}}\label{proof:existence}
\textbf{Theorem \ref{thm:existence}.} Assume\newline (i) the domain $\Omega$ satisfies the \emph{uniform exterior sphere condition}, i.e., 
\[
\exists r_0>0, \forall \bm{x}\in\partial\Omega, \textrm{we have} \left\{\bm{n}\big| \|\bm{n}\|=1, B(\bm{x}-r_0\bm{n},r_0)\cap\Omega=\emptyset\right\}\neq \emptyset,
\]
where $B(\bm{x},r)$ is a open ball with center $\bm{x}$ and radius $r$;
\newline
(ii) the domain $\Omega$ satisfies the \emph{uniform cone condition}, i.e.,
\[
\exists \delta>0, \alpha\in[0,1), \forall \bm{x}\in\partial\Omega, \textrm{we have} \left\{\bm{m}\big|\|\bm{m}\|=1, \forall \bm{y}\in B(\bm{x},\delta)\cap\partial\Omega, C(\bm{y},\bm{m},\alpha)\cup B(\bm{x},\delta)\subset\Omega\right\} \neq \emptyset,
\]
where $C(\bm{y},\bm{m},\alpha)=\left\{\bm{z}\big| (\bm{z}-\bm{y})\cdot \bm{m}\geq \|\bm{z}-\bm{y}\|\alpha \right\}$.
\newline
(iii) the velocity field $\bm{v}_t(\bm{x})$ is Lipschitz continuous in $\bm{x}\in\Omega$ (uniformly on $t$).
\newline Then the solution to the reflected ODE \eqref{eq:reflected-particle-ODE} exists and is unique on $t\in[0,1]$.

\paragraph{Sketch of Proof}
The reflected ODE we considered can be viewed as a special case of reflected SDE where the diffusion coefficient is zero.
Therefore, Theorem \ref{thm:existence} is indeed a corollary of \citet[Theorem 2.5.4]{pilipenko2014reflectedSDE} with no diffusion term (i.e., diffusion coefficients $b_k(\bm{x}_t)$, $1\leq k\leq m$, are zeros).

For a more rigorous derivation, note that \citet[Theorem 2.5.2]{pilipenko2014reflectedSDE} says that under the assumption (i) (ii) of the domain $D$ in our Theorem \ref{thm:existence}, there exists a unique solution to the Skorokhod problem with normal reflection as long as the drift function $a(\bm{x}_t)$ is continuous.
To prove the existence and uniqueness of the solution to the reflected ODE, one can first use Euler discretization to find a sequence of approximate solutions whose existence and uniqueness are guaranteed by \citet[Theorem 2.5.2]{pilipenko2014reflectedSDE}. 
We can then prove that this sequence of approximate solutions converges to the solution of the reflected ODE problem as the stepsize goes to zero. 
This proof simply follows the proof provided by \citet{pilipenko2014reflectedSDE} for reflected SDEs.

\subsection{Proof of Theorem \ref{thm:crfm}}\label{proof:crfm}
First of all, we have
\begin{subequations}
\footnotesize
\begin{align}
\|\bm{v}_{\bm{\theta}}(\bm{x}, t)-\bm{v}_t(\bm{x})\|^2&= \|\bm{v}_{\bm{\theta}}(\bm{x}, t)\|^2-2\bm{v}_{\bm{\theta}}(\bm{x}, t)\cdot\bm{v}_t(\bm{x})+\|\bm{v}_t(\bm{x})\|^2\\
\|\bm{v}_{\bm{\theta}}(\bm{x}, t)-\bm{v}_t(\bm{x}|\bm{x}_1)\|^2&=\|\bm{v}_{\bm{\theta}}(\bm{x}, t)\|^2-2\bm{v}_{\bm{\theta}}(\bm{x}, t)\cdot\bm{v}_t(\bm{x}|\bm{x}_1)+\|\bm{v}_t(\bm{x}|\bm{x}_1)\|^2
\end{align}
\end{subequations}
To prove that $\nabla \mathcal{L}_{\textrm{RFM}}(\bm{\theta}) = \nabla \mathcal{L}_{\textrm{CRFM}}(\bm{\theta})$, it can be easily seen that
\begin{equation}
\mathbb{E}_{p_{\mathrm{data}}(\bm{x}_1)p_t(\bm{x}|\bm{x}_1)}\|\bm{v}_{\bm{\theta}}(\bm{x}, t)\|^2=\mathbb{E}_{p_t(\bm{x})}\|\bm{v}_{\bm{\theta}}(\bm{x}, t)\|^2
\end{equation}
by the definition of $p_t(\bm{x})$ in equation \eqref{eq:marginal-probability-path}.
The remaining thing to show is
\begin{equation}
\mathbb{E}_{p_{\mathrm{data}}(\bm{x}_1)p_t(\bm{x}|\bm{x}_1)}\bm{v}_{\bm{\theta}}(\bm{x}, t)\cdot\bm{v}_t(\bm{x}|\bm{x}_1)=
\mathbb{E}_{p_t(\bm{x})}\bm{v}_{\bm{\theta}}(\bm{x}, t)\cdot\bm{v}_t(\bm{x}).
\end{equation}
By the definition of $\bm{v}_t(\bm{x})$ in equation \eqref{eq:marginal-velocity-field}, we have
\begin{align*}
\mathbb{E}_{p_t(\bm{x})}\bm{v}_{\bm{\theta}}(\bm{x}, t)\cdot\bm{v}_t(\bm{x})&=\int_\Omega p_t(\bm{x})\bm{v}_{\bm{\theta}}(\bm{x}, t)\cdot\bm{v}_t(\bm{x})\dif\bm{x}\\
&=\int_\Omega p_t(\bm{x})\bm{v}_{\bm{\theta}}(\bm{x}, t)\cdot \int_\Omega \bm{v}_t(\bm{x}|\bm{x}_1)\frac{p_t(\bm{x}|\bm{x}_1)p_{\mathrm{data}}(\bm{x}_1)}{p_t(\bm{x})}\dif \bm{x}_1\dif\bm{x}\\
&=\int_\Omega\int_\Omega  \bm{v}_{\bm{\theta}}(\bm{x}, t)\cdot \bm{v}_t(\bm{x}|\bm{x}_1)p_t(\bm{x}|\bm{x}_1)p_{\mathrm{data}}(\bm{x}_1)\dif \bm{x}_1\dif\bm{x}\\
&=\mathbb{E}_{p_{\mathrm{data}}(\bm{x}_1)p_t(\bm{x}|\bm{x}_1)}\bm{v}_{\bm{\theta}}(\bm{x}, t)\cdot\bm{v}_t(\bm{x}|\bm{x}_1).
\end{align*}
This finishes the proof.

\subsection{Proof of Theorem \ref{thm:wasserstein}}\label{proof:wasserstein}
We first give an alternative expression \citep{pilipenko2014reflectedSDE} of the reflected ODE
\begin{subequations}\label{eq:reflected-ode-app}
\arraycolsep=1.8pt
\def\arraystretch{1.5}
\begin{align}
\mathrm{d}\bm{\phi}_t(\bm{x}) &= \bm{v}_t(\bm{\phi}_t(\bm{x}))\mathrm{d}t + \bm{n}(\bm{\phi}_t(\bm{x}))\dif l_t,\\
\bm{\phi}_0(\bm{x}) &= \bm{x},
\end{align}
\end{subequations}
where the initial point $\bm{x}\sim p_0(\bm{x})$, $\bm{n}(\bm{x})$ is the  inward unit normal vector at $\bm{x}$ on $\partial\Omega$, and $l_t$ is non-decreasing in $t$ with $l_0=0$.
Moreover, $l_t$ satisfies $\int_{0}^t \mathbb{I}(\bm{\phi}_s(\bm{x})\notin\partial\Omega)\dif l_s=0$ for $t>0$, i.e., $\dif l_t$ vanishes in the interior and will push the velocity along $\bm{n}(\bm{x})$ when the trajectory hits the boundary.

Let $\bm{\phi}_{\bm{\theta},t}(\bm{x})$ (with a reflection term $\bar{l}_t$) be the solution to the reflected ODE \eqref{eq:reflected-ode-app} with velocity field $\bm{v}_{\bm{\theta}}(\bm{x},t)$.
By the definition of Wasserstein-2 distance, we have 
\begin{equation}\label{eq:proof-1}
W_2^2\left(p_{1}(\bm{x}),p_{\bm{\theta},1}(\bm{x})\right)\leq \int_{\Omega}\|\bm{\phi}_1(\bm{x})-\bm{\phi}_{\bm{\theta},1}(\bm{x})\|^2 p_0(\bm{x})\dif \bm{x}:= \hat{W}_2^2\left(p_{t}(\bm{x}),p_{\bm{\theta},t}(\bm{x})\right)\big|_{t=1}.
\end{equation}
Take the derivative of $\hat{W}_2^2\left(p_{t}(\bm{x}),p_{\bm{\theta},t}(\bm{x})\right)$ w.r.t. $t$ gives
\begin{equation}
\dif \hat{W}_2^2\left(p_{t}(\bm{x}),p_{\bm{\theta},t}(\bm{x})\right)
=2\int_{\Omega}\left(\bm{\phi}_t(\bm{x})-\bm{\phi}_{\bm{\theta},t}(\bm{x})\right)\left(\dif\bm{\phi}_t(\bm{x})-\dif\bm{\phi}_{\bm{\theta},t}(\bm{x})\right) p_0(\bm{x})\dif \bm{x}
\end{equation}
Note that
\begin{align*}
&\left(\bm{\phi}_t(\bm{x})-\bm{\phi}_{\bm{\theta},t}(\bm{x})\right)\left(\dif\bm{\phi}_t(\bm{x})-\dif\bm{\phi}_{\bm{\theta},t}(\bm{x})\right)\\
=& \left(\bm{\phi}_t(\bm{x})-\bm{\phi}_{\bm{\theta},t}(\bm{x})\right)\left[(\bm{v}_t(\bm{\phi}_t(\bm{x}))-\bm{v}_{\bm{\theta}}(\bm{\phi}_{\bm{\theta},t}(\bm{x}),t))\mathrm{d}t + \bm{n}(\bm{\phi}_t(\bm{x}))\dif l_t-\bm{n}(\bm{\phi}_{\bm{\theta},t}(\bm{x}))\dif \bar{l}_t\right]\\
=& \left(\bm{\phi}_t(\bm{x})-\bm{\phi}_{\bm{\theta},t}(\bm{x})\right)(\bm{v}_t(\bm{\phi}_t(\bm{x}))-\bm{v}_{\bm{\theta}}(\bm{\phi}_{\bm{\theta},t}(\bm{x}),t))\mathrm{d}t \\
&\quad\quad\quad\quad\quad\quad+ \left(\bm{\phi}_t(\bm{x})-\bm{\phi}_{\bm{\theta},t}(\bm{x})\right)\bm{n}(\bm{\phi}_t(\bm{x}))\dif l_t - \left(\bm{\phi}_t(\bm{x})-\bm{\phi}_{\bm{\theta},t}(\bm{x})\right)\bm{n}(\bm{\phi}_{\bm{\theta},t}(\bm{x}))\dif \bar{l}_t,
\end{align*}
The first term can be bounded by
\begin{align*}
&\left(\bm{\phi}_t(\bm{x})-\bm{\phi}_{\bm{\theta},t}(\bm{x})\right)(\bm{v}_t(\bm{\phi}_t(\bm{x}))-\bm{v}_{\bm{\theta}}(\bm{\phi}_{\bm{\theta},t}(\bm{x}),t))\\
=&\left(\bm{\phi}_t(\bm{x})-\bm{\phi}_{\bm{\theta},t}(\bm{x})\right)(\bm{v}_t(\bm{\phi}_t(\bm{x}))-\bm{v}_{\bm{\theta}}(\bm{\phi}_t(\bm{x}),t))+\left(\bm{\phi}_t(\bm{x})-\bm{\phi}_{\bm{\theta},t}(\bm{x})\right)(\bm{v}_{\bm{\theta}}(\bm{\phi}_t(\bm{x}),t)-\bm{v}_{\bm{\theta}}(\bm{\phi}_{\bm{\theta},t}(\bm{x}),t))\\
\leq & \frac{1}{2}\|\bm{\phi}_t(\bm{x})-\bm{\phi}_{\bm{\theta},t}(\bm{x})\|^2+\frac{1}{2}\|\bm{v}_t(\bm{\phi}_t(\bm{x}))-\bm{v}_{\bm{\theta}}(\bm{\phi}_t(\bm{x}),t)\|^2 + M\|\bm{\phi}_t(\bm{x})-\bm{\phi}_{\bm{\theta},t}(\bm{x})\|^2
\end{align*}
where in the last inequality we use the mean inequality and the fact that $\bm{v}_{\bm{\theta}}(\bm{x},t)$ is $M$-Lipschitz in $\bm{x}$.
To handle the second term, note that $dl_t>0$ only if $\bm{\phi}_t(\bm{x})\in\partial\Omega$. When $\bm{\phi}_t(\bm{x})\in\partial\Omega$, we know 
$\left(\bm{\phi}_t(\bm{x})-\bm{\phi}_{\bm{\theta},t}(\bm{x})\right)\bm{n}(\bm{\phi}_t(\bm{x}))\leq 0$ due to the convexity of $\Omega$.
Therefore, the second term satisfies
\[
\left(\bm{\phi}_t(\bm{x})-\bm{\phi}_{\bm{\theta},t}(\bm{x})\right)\bm{n}(\bm{\phi}_t(\bm{x}))\dif l_t\leq 0.
\]
Using the same argument, we know that the third term satisfies
\[
-\left(\bm{\phi}_t(\bm{x})-\bm{\phi}_{\bm{\theta},t}(\bm{x})\right)\bm{n}(\bm{\phi}_{\bm{\theta},t}(\bm{x}))\dif \bar{l}_t\leq 0
\]
See \citet{lamperski2021projected} for a more rigorous argument.
Therefore, we have the following bound for $\dif \hat{W}_2^2\left(p_{t}(\bm{x}),p_{\bm{\theta},t}(\bm{x})\right)/\dif t$
\begin{align*}
&\dif \hat{W}_2^2\left(p_{t}(\bm{x}),p_{\bm{\theta},t}(\bm{x})\right)/\dif t\\
\leq & \int_{\Omega}\|\bm{\phi}_t(\bm{x})-\bm{\phi}_{\bm{\theta},t}(\bm{x})\|^2p_0(\bm{x})\dif \bm{x}+\int_{\Omega}\|\bm{v}_t(\bm{\phi}_t(\bm{x}))-\bm{v}_{\bm{\theta}}(\bm{\phi}_t(\bm{x}),t)\|^2p_0(\bm{x})\dif \bm{x} + 2M\int_{\Omega}\|\bm{\phi}_t(\bm{x})-\bm{\phi}_{\bm{\theta},t}(\bm{x})\|^2
p_0(\bm{x})\dif \bm{x}\\
=&(1+2M)\hat{W}_2^2\left(p_{t}(\bm{x}),p_{\bm{\theta},t}(\bm{x})\right)+\int_{\Omega}\|\bm{v}_t(\bm{\phi}_t(\bm{x}))-\bm{v}_{\bm{\theta}}(\bm{\phi}_t(\bm{x}),t)\|^2p_0(\bm{x})\dif \bm{x}.
\end{align*}
By Gronwall's inequality and $\hat{W}_2^2\left(p_{t}(\bm{x}),p_{\bm{\theta},t}(\bm{x})\right)\big|_{t=0}=0$, we have
\begin{equation}\label{eq:proof-2}
\hat{W}_2^2\left(p_{t}(\bm{x}),p_{\bm{\theta},t}(\bm{x})\right) \leq e^{1+2M}\int_0^1\int_{\Omega}\|\bm{v}_t(\bm{\phi}_t(\bm{x}))-\bm{v}_{\bm{\theta}}(\bm{\phi}_t(\bm{x}),t)\|^2p_0(\bm{x})\dif \bm{x}\dif t=e^{1+2M}\mathcal{L}_{\textrm{RFM}}(\bm{\theta}).
\end{equation}
By combining equation \eqref{eq:proof-1} and \eqref{eq:proof-2}, we finally conclude
\[
W_2^2\left(p_{1}(\bm{x}),p_{\bm{\theta},1}(\bm{x})\right)\leq e^{1+2M}\mathcal{L}_{\textrm{RFM}}(\bm{\theta}).
\]
This finishes the proof.

\subsection{Proof of Theorem \ref{thm:guided-velocity-field}}\label{proof:guided-velocity-field}

Let $\Omega=\mathbb{R}^d$ and the prior distribution $p_0(\cdot)$ be the standard Gaussian distribution.
Under the OT conditional flow, the conditional velocity field is $\bm{v}_t(\bm{x}|\bm{x}_1)=\frac{\bm{x}_1-(1-\sigma_{\min})\bm{x}}{1-(1-\sigma_{\min})t}$, and the conditional probability path is $p_{t}(\bm{x}|\bm{x}_1)=\mathcal{N}(t\bm{x}_1,(1-(1-\sigma_{\min})t)^2\bm{I})$.
Define
\[
\bar{p}_t(\bm{x}_1|\bm{x}) = \frac{p_t(\bm{x}|\bm{x}_1)p_{\textrm{data}}(\bm{x}_1)}{p_t(\bm{x})}.
\]
By the definition of $\bm{v}_t(\bm{x})$,
we have
\begin{align*}
\bm{v}_t(\bm{x})&=\mathbb{E}_{\bar{p}_t(\bm{x}_1|\bm{x})}\bm{v}_t(\bm{x}|\bm{x}_1)\\
&=\mathbb{E}_{\bar{p}_t(\bm{x}_1|\bm{x})}\frac{\bm{x}_1-(1-\sigma_{\min})\bm{x}}{1-(1-\sigma_{\min})t}\\
&=\frac{\bm{x}}{t}-\frac{1-(1-\sigma_{\min})t}{t}\mathbb{E}_{\bar{p}_t(\bm{x}_1|\bm{x})}\frac{\bm{x}-t\bm{x}_1}{(1-(1-\sigma_{\min})t)^2}\\
&=\frac{\bm{x}}{t}+\frac{1-(1-\sigma_{\min})t}{t}\mathbb{E}_{\bar{p}_t(\bm{x}_1|\bm{x})}\nabla\log p_{t}(\bm{x}|\bm{x}_1)\\
&=\frac{\bm{x}}{t}+\frac{1-(1-\sigma_{\min})t}{t}\nabla\log p_{t}(\bm{x})
\end{align*}
where we use the fact
\[
\mathbb{E}_{\bar{p}_t(\bm{x}_1|\bm{x})} \nabla_{\bm{x}}\log p_{t}(\bm{x}|\bm{x}_1)=\int \frac{p_t(\bm{x}|\bm{x}_1)p_{\textrm{data}}(\bm{x}_1)}{p_{t}(\bm{x})}\frac{\nabla_{\bm{x}} p_{t}(\bm{x}|\bm{x}_1)}{p_{t}(\bm{x}|\bm{x}_1)}\dif \bm{x}_1=\frac{1}{p_{t}(\bm{x})}\nabla_{\bm{x}} \int  p_{t}(\bm{x}|\bm{x}_1)p_{\mathrm{data}}(\bm{x}_1)\dif \bm{x}_1=\nabla_{\bm{x}}\log p_{t}(\bm{x}).
\]
This finishes the proof.

\section{Implementation Details}\label{app:implementation-details}
\subsection{Low-dimensional Toy Examples}\label{app:implimentation-details-toy}
For all data sets, we use 6-layer MLPs with 512 channels for the parametrization of the velocity model $\bm{v}_{\bm{\theta}}(\bm{x},t)$.
We use the sinusoidal positional embedding \citep{vaswani2017transformer}  with 512 channels for the time step $t$ and adds the time embeddings to the input of each activition function in MLPs.
We add a residual block after each linear layer in MLPs.
All the activition function is set to be the Gaussian error linear units (GELU) \citep{hendrycks2016GELU}.
All models are implemented in PyTorch \citep{Paszke2019PyTorchAI} and optimized with the Adam \citep{ADAM} optimizer ($\beta_1=0.9,\beta_2=0.999$).
The learning rate is set to be 0.0003 at the beginning, with a decay rate of 0.75 per 10,000 iterations.
The results are collected after 200,000 iterations.

\begin{table}[t]
\centering 
\caption{Hyper-parameters used for training each model.}
\begin{tabular}{l cc}
\toprule
Datasets                          &CIFAR-10      &ImageNet64 \\
\midrule
Channels                          &128           &192 \\ 
Depth                             &2             &3 \\ 
Channels multiple                 &1,2,2,2       &1,2,3,4  \\
Heads                             &4             &4\\ 
Heads Channels                    &64            &64 \\
Attention resolution              &16            &32,16,8 \\
Dropout                           &0.1           &0.1 \\
Effective Batch size              &128           &2048 \\
Iterations                        &800k          &540k  \\
Learning Rate                     &2e-4          &1e-4 \\
EMA rate                          &0.9999        &0.9999 \\
Learning Rate Scheduler           &Constant      &Constant \\
Warmup Steps                      &5k            &5k \\
\bottomrule
\end{tabular}
\label{trainingsetting}
\end{table}

\subsection{Unconditional Image Generation}\label{app:cifar10}
The data are rescaled to a hypercube $\Omega=[-1,1]^{3\times 32\times 32}$.
For FM, we set the prior distribution to be a standard Gaussian distribution and use the OT conditional velocity field.
For RFM, we set the prior distribution to be a truncated standard Gaussian distribution over $\Omega$ and use the OT conditional velocity field in equation \eqref{eq:OT-probability-paths}.
The architecture of the velocity model $\bm{v}_{\bm{\theta}}(\bm{x},t)$ in RFM and FM is the same as the UNet \citep{ronneberger2015unet} of score function model in \citet{dhariwal2021diffusion} (see Table \ref{trainingsetting} for details).
We use full 32-bit precision for training on CIFAR-10. 
Following \citet{lipman2023FM}, the velocity model in RFM is optimized with Adam \citep{ADAM} optimizer ($\beta_{1} = 0.9$, $\beta_{2} = 0.999$, $\textrm{weight decay} = 0.0$, and $\epsilon=10^{-8}$) and a constant learning rate of 0.0002 after a warm-up phase of 5000 training steps.
In the warm-up phase, the learning rate is linearly increased from $10^{-8}$ to the maximum learning rate 0.0002.
The results are collected after 800,000 iterations with a batch size of 128.
We use the Dormand–Prince method \citep{dormand1980dopri} with absolute and relative tolerances of $10^{-5}$ to simulate the reflected CNFs.
It costs 1.5 day on 8 Nvidia 2080 Ti GPUs to train reflected CNFs with RFM on CIFAR-10. 

\subsection{Conditional Image Generation}\label{app:imagenet64}
Similarly to the unconditional image generation tasks, we rescale the data to a hypercube $\Omega=[-1,1]^{3\times 64\times 64}$.
The prior distribution is set to be the standard Gaussian distribution for FM and the truncated standard Gaussian distribution for RFM.
Both FM and RFM use the OT conditional velocity field.
The class-conditioned velocity field model $\bm{v}_{\bm{\theta}}(\bm{x},t,c)$ in RFM and FM is the same as the UNet in the class-conditioned score function model in \citet{dhariwal2021diffusion} (see Table \ref{trainingsetting}).
We use 16-bit mixed precision for training on ImageNet64.
We train FM and RFM using AdamW optimizer \citep{loshchilov2018adamw} with a constant learning rate of 0.0001 after a warm-up phase of 5000 steps.
In the warm-up phase, the learning rate is linearly increased from $10^{-8}$ to the maximum learning rate 0.0001.
The results are collected after 540,000 iterations with a batch size of 2048.
For both FM and RFM, we use the Dormand–Prince method \citep{dormand1980dopri} with absolute and relative tolerances of $10^{-5}$ to simulate the (reflected) CNFs.
It costs 14 days on 32 Nvidia A100 GPUs to train FM and RFM on ImageNet $(64\times 64)$.

\section{Addtional Experimental Results}\label{app:exp}

\subsection{Low-dimensional Toy Examples}\label{app:exp-toy-example}
For the four 2D constrained domains: hypercube, simplex, half annulus, and cup, we plot the velocity field in reflected CNFs learned by RFM in Figure \ref{fig:velocity-rfm} and the velocity field in CNFs learned by FM in Figure \ref{fig:velocity-fm}.
We also report the KL divergences to the ground truth (Table \ref{tab:KL2D}) and the constraint violation ratio (Table \ref{tab:outier2D}) of samples generated using the Dormand–Prince method \citep{dormand1980dopri} with absolute and relative tolerances of $10^{-5}$. According to Table \ref{tab:2D-NFE}, we see that compared to FM, RFM uses similar NFE on hypercube and simplex, but uses more NFE on half annulus and cup due to the non-optimal transport conditional velocity field.

We further investigate the effect of discretization steps on the constraint violation ratio (Table \ref{tab:outier-ablation}). Table \ref{tab:outier-ablation} shows that the constraint violation ratio can increase when NFE gets larger.
We attribute this to the more time the solver spends exploring around the boundary. 
When the NFE is small, the samples generated by FM collapse to the interior of the domain, which causes a low constraint violation rate but poor approximation accuracy.

\begin{table}[h]
\centering 
\caption{KL divergences to the ground truth obtained by different methods on low-dimensional generation tasks.
For each method, we calculate the KL divergence using the Python ITE module \citep{szabo2014ITE} with a sample size 50,000.
We generate the samples from CNFs and reflected CNFs using the Dormand–Prince method \citep{dormand1980dopri} with absolute and relative tolerances of $10^{-5}$.
The results are averaged over 10 independent runs with standard deviation in the brackets.
}
\label{tab:KL2D}
\vspace{0.5em}
\begin{tabular}{lcccccc}
\toprule
\multirow{2}{*}{Method}&\multicolumn{2}{c}{Hypercube}&\multicolumn{2}{c}{Simplex}&Half Annulus&Cup \\ 
\cmidrule(l){2-3}\cmidrule(l){4-5}\cmidrule(l){6-6}\cmidrule(l){7-7}
&$d=2$&$d=10$&$d=2$&$d=10$&$d=2$&$d=2$\\
\midrule
RDM&0.0110(0.0003)&0.4798(0.0006)&0.0055(0.0001)&0.7613(0.0036)& N/A &N/A\\
FM&0.0049(0.0008)&\textbf{0.0247(0.0028)}&0.0093(0.0018)&0.0563(0.0012)&0.0091(0.0008) &0.0132(0.0018)\\
RFM&\textbf{0.0020(0.0013)}&0.0248(0.0027)&\textbf{0.0030(0.0012)}&\textbf{0.0460(0.0013)} &\textbf{0.0032(0.0014)}&\textbf{0.0069(0.0019)}\\
\bottomrule
\end{tabular}
\end{table}

\begin{table}[h]
\centering 
\setlength\tabcolsep{3pt}
\caption{Constraint violation ratio (\textperthousand) of different methods on low-dimensional generation tasks.
We generate the samples from CNFs and reflected CNFs using the Dormand–Prince method \citep{dormand1980dopri} with absolute and relative tolerances of $10^{-5}$.
}
\label{tab:outier2D}
\vspace{0.5em}
\begin{tabular}{lcccccc}
\toprule
\multirow{2}{*}{Method}&\multicolumn{2}{c}{Hypercube}&\multicolumn{2}{c}{Simplex}&Half Annulus&Cup \\ 
\cmidrule(l){2-3}\cmidrule(l){4-5}\cmidrule(l){6-6}\cmidrule(l){7-7}
&$d=2$&$d=10$&$d=2$&$d=10$&$d=2$&$d=2$\\
\midrule
FM&3.2&28.6&12.1&92.8& 10.4 &15.1\\
RFM&0.0&0.0&0.0&0.0&0.0&0.0\\
\bottomrule
\end{tabular}
\end{table}

\begin{table}[h]
    \centering
    \caption{NFE used to samples from different models on low-dimensional generation tasks using the Dormand–Prince method \citep{dormand1980dopri} with absolute and relative tolerances of $10^{-5}$.}
    \label{tab:2D-NFE}
\vspace{0.5em}
\begin{tabular}{lcccccc}
\toprule
\multirow{2}{*}{Method}&\multicolumn{2}{c}{Hypercube}&\multicolumn{2}{c}{Simplex}&Half Annulus&Cup \\ 
\cmidrule(l){2-3}\cmidrule(l){4-5}\cmidrule(l){6-6}\cmidrule(l){7-7}
&$d=2$&$d=10$&$d=2$&$d=10$&$d=2$&$d=2$\\
\midrule
FM&103&92&111&151&110&119\\
RFM&117&112&116&99&206&275\\
\bottomrule
\end{tabular}
\end{table}

\begin{table}[h!]
\centering 
\setlength\tabcolsep{3pt}
\caption{Constraint violation ratio (\textperthousand) and KL divergence with varying NFE on the generation task on Simplex ($d=2$). Solver: third-order Heun algorithm.
}
\label{tab:outier-ablation}
\vspace{0.5em}
\begin{tabular}{lccc}
\toprule
NFE & 15& 60&	300\\
\midrule
Constraint violation ratio&	0.5&	1.7&	9.0\\
KL divergence&	0.1805(0.0024)&	0.0084(0.0013)&	0.0062(0.0011)\\
\bottomrule
\end{tabular}
\end{table}

\begin{figure}[t]
    \centering
    \includegraphics[width=\linewidth]{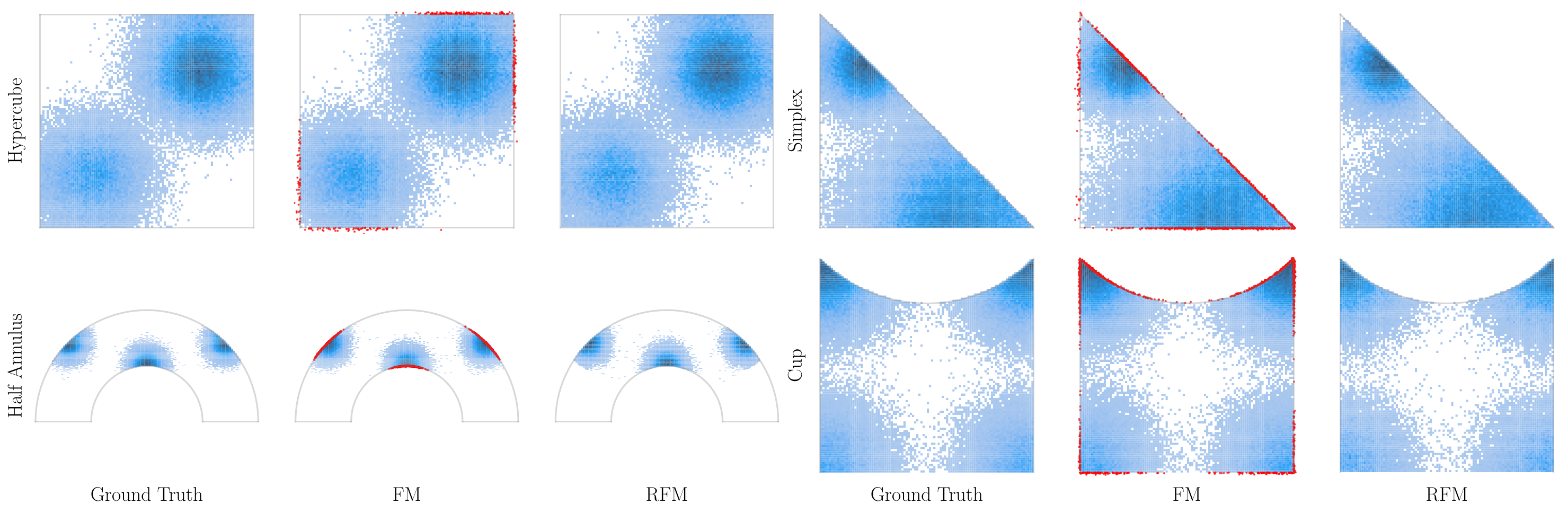}
    \caption{The histplots of samples obtained by different methods compared to the ground truth on the two-dimensional hypercube, simplex, and cup data set.
    We generate the samples from CNFs and reflected CNFs using the Dormand–Prince method \citep{dormand1980dopri} with absolute and relative tolerances of $10^{-5}$.
    Samples out of the constrained domain are plotted with red dots.
    The total sample size is 100,000.
    }
    \label{fig:histplot2D}
\end{figure}

\begin{figure}[t]
    \centering
    \includegraphics[width=\linewidth]{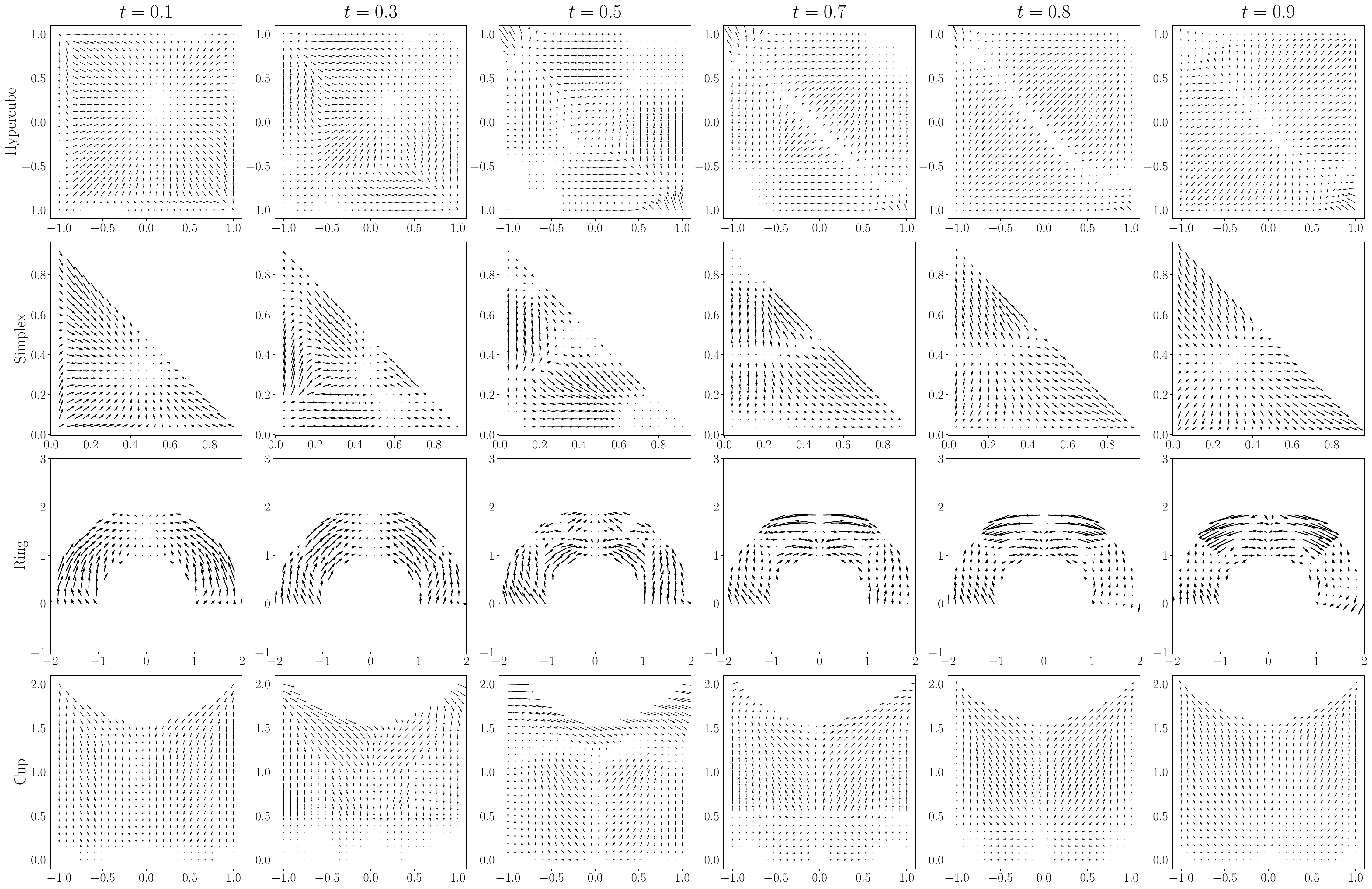}
    \caption{Velocity field in reflected CNFs learned by RFM for different $t$s on two-dimensional generation tasks.}
    \label{fig:velocity-rfm}
\end{figure}

\begin{figure}[t]
    \centering
    \includegraphics[width=\linewidth]{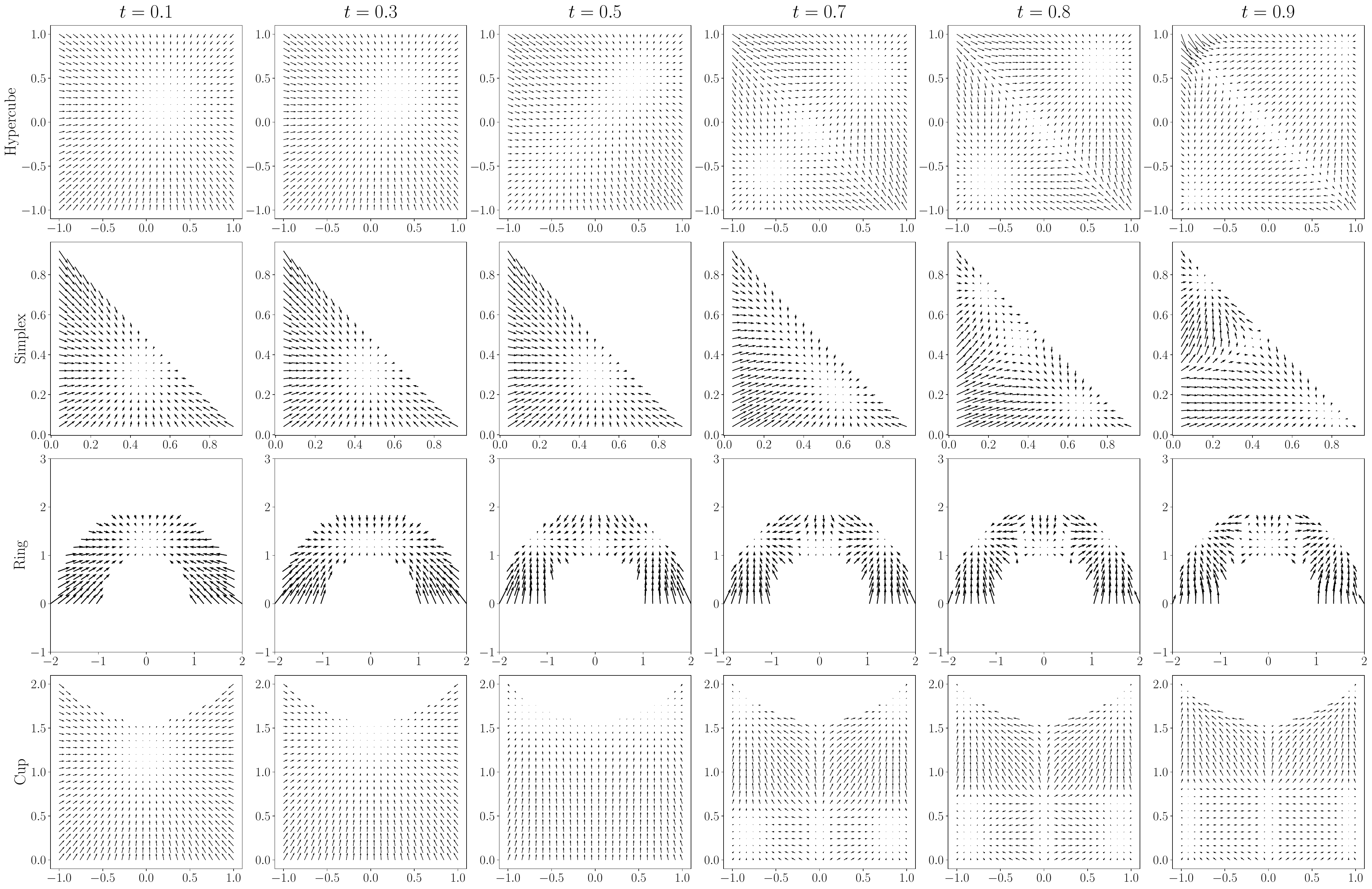}
    \caption{Velocity field in CNFs learned by FM for different $t$s on two-dimensional generation tasks.
    We only plot those velocity fields in the constrained domain.}
    \label{fig:velocity-fm}
\end{figure}

\subsection{Unconditional Image Generation}\label{app:cifar-results}
We investigate that the effect of the ODE solvers on the FID score in Table \ref{tab:cifar-solver-ablation}, where the DOPRI5 achieves the best result.
Samples generated by the reflected CNFs trained with RFM on CIFAR-10 are shown in Figure \ref{fig:cifar10}.
We also test the uniform distribution over the constrained domain as the prior distribution for RFM in Figure \ref{fig:cifar10-uniform}.

\begin{table}[h!]
\centering 
\setlength\tabcolsep{3pt}
\caption{FID score obtained by RFM with different ODE solvers \citep{chen2018neuralode} on CIFAR10.
}
\label{tab:cifar-solver-ablation}
\vspace{0.5em}
\begin{tabular}{lccccc}
\toprule
Solver&	Dopri5 (NFE=139)&	RK4 (NFE=100)&	RK4 (NFE=200)&	Heun3 (NFE=150)&	Heun3 (NFE=300)\\
\midrule
FID&	\textbf{4.76}&	5.37&	4.85&	4.92&	4.81\\
\bottomrule
\end{tabular}
\end{table}

\begin{figure}[t]
    \centering
    \includegraphics[width=\linewidth]{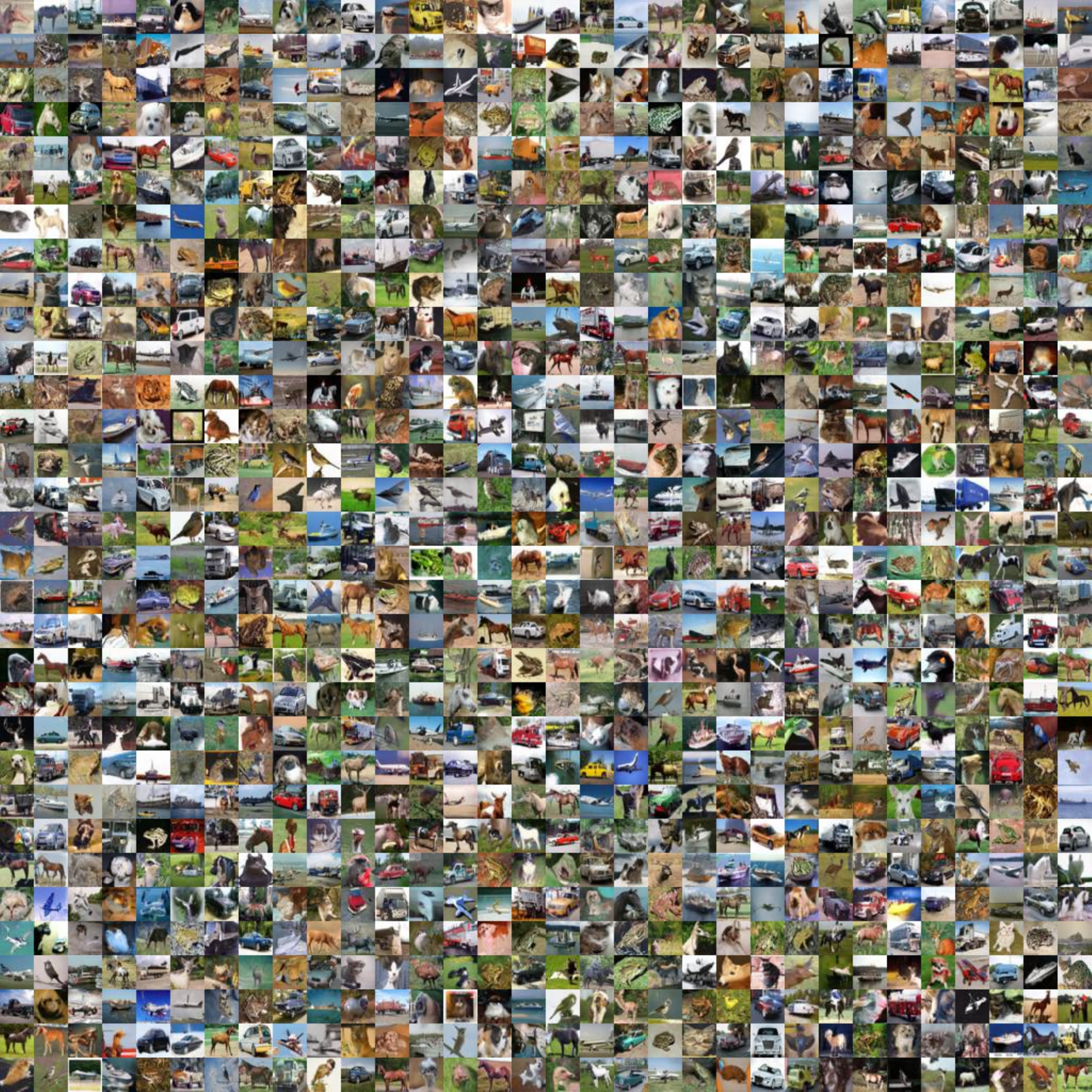}
    \caption{Samples generated by the reflected CNFs trained with RFM on CIFAR-10 ($\textrm{NFE}=139$). The prior distribution is set to the truncated standard Gaussian distribution over $[-1,1]^{3\times 32\times 32}$.
    The ODE solver is the Dormand–Prince method \citep{dormand1980dopri} with absolute and relative tolerances of $10^{-5}$.
    }
    \label{fig:cifar10}
\end{figure}

\begin{figure}[t]
    \centering
    \includegraphics[width=\linewidth]{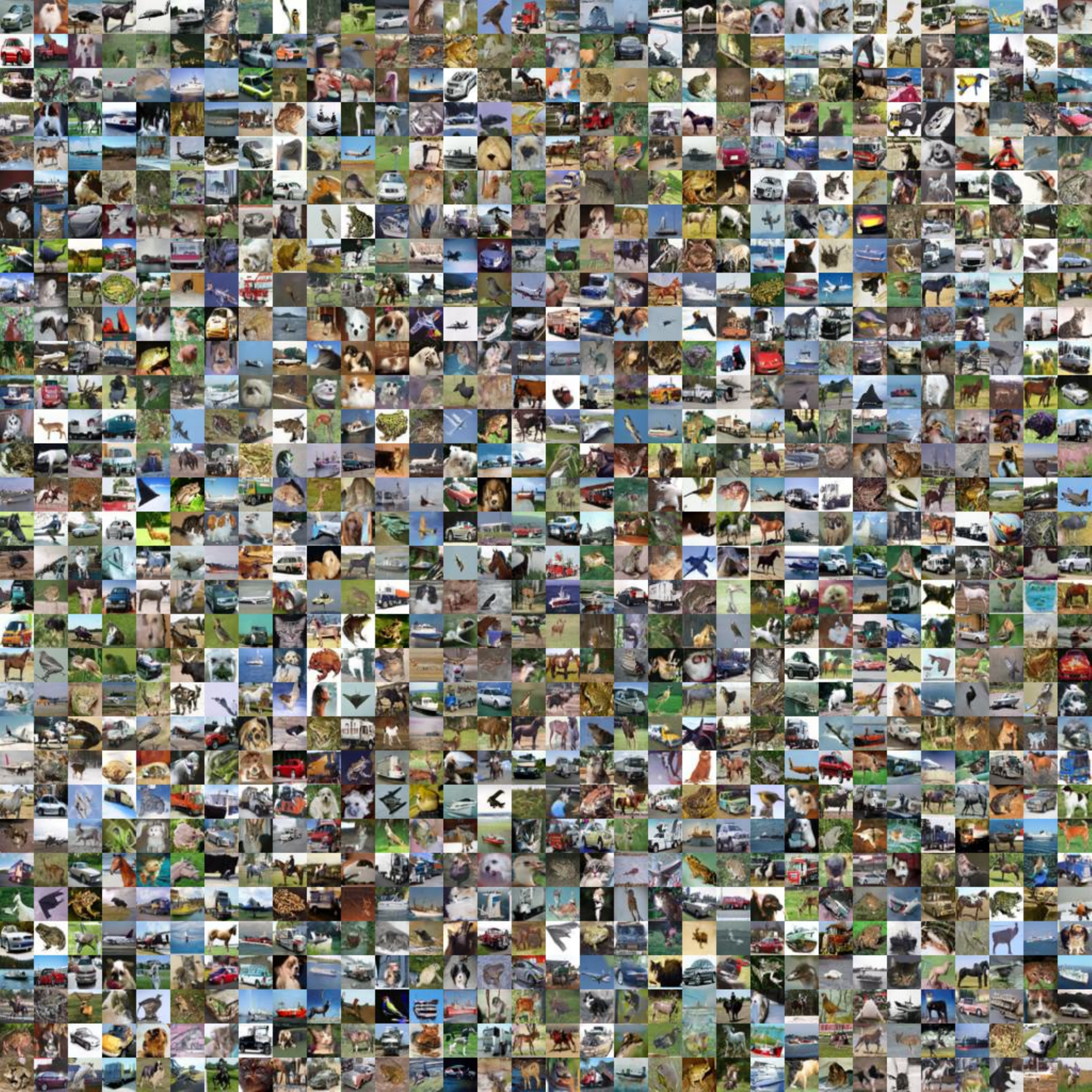}
    \caption{Samples generated by the reflected CNFs trained with RFM on CIFAR-10 ($\textrm{NFE}=118$). The prior distribution is set to the uniform distribution over $[0,1]^{3\times 32\times 32}$.
    The ODE solver is the Dormand–Prince method \citep{dormand1980dopri} with absolute and relative tolerances of $10^{-5}$.
    }
    \label{fig:cifar10-uniform}
\end{figure}

\subsection{Conditional Image Generation}\label{app:imagenet-results}
Table \ref{tab:imagenet-fid} reports the FID scores on ImageNet ($64\times 64$) conditional generation benchmark.
Generally, the FID score of class-conditioned images under high guidance weight can be poor and the generated images (Figure \ref{fig:imagenet64}) are more informative.
We provide a sample quality comparison between FM, RFM, and RDM in Figure \ref{fig:imagenet-appendix}.

\begin{table}[h!]
\centering 
\setlength\tabcolsep{3pt}
\caption{FID scores on ImageNet ($64\times  64$) conditional generation benchmark (flow guidance weight $w=15$). Solver: Dopri5 for FM/RFM, RK45 for RDM(ODE), Euler-Maruyama for RDM(SDE).
}
\label{tab:imagenet-fid}
\vspace{0.5em}
\begin{tabular}{lcccc}
\toprule
Method&	FM&	RDM (ODE)&	RDM (SDE)&	RFM\\
\midrule
NFE&	769&	871&	1000&	723\\
FID&	31.83&	66.74&	35.70&	\textbf{26.67}\\
\bottomrule
\end{tabular}
\end{table}

\begin{figure}[t]
    \centering
    \includegraphics[width=\linewidth]{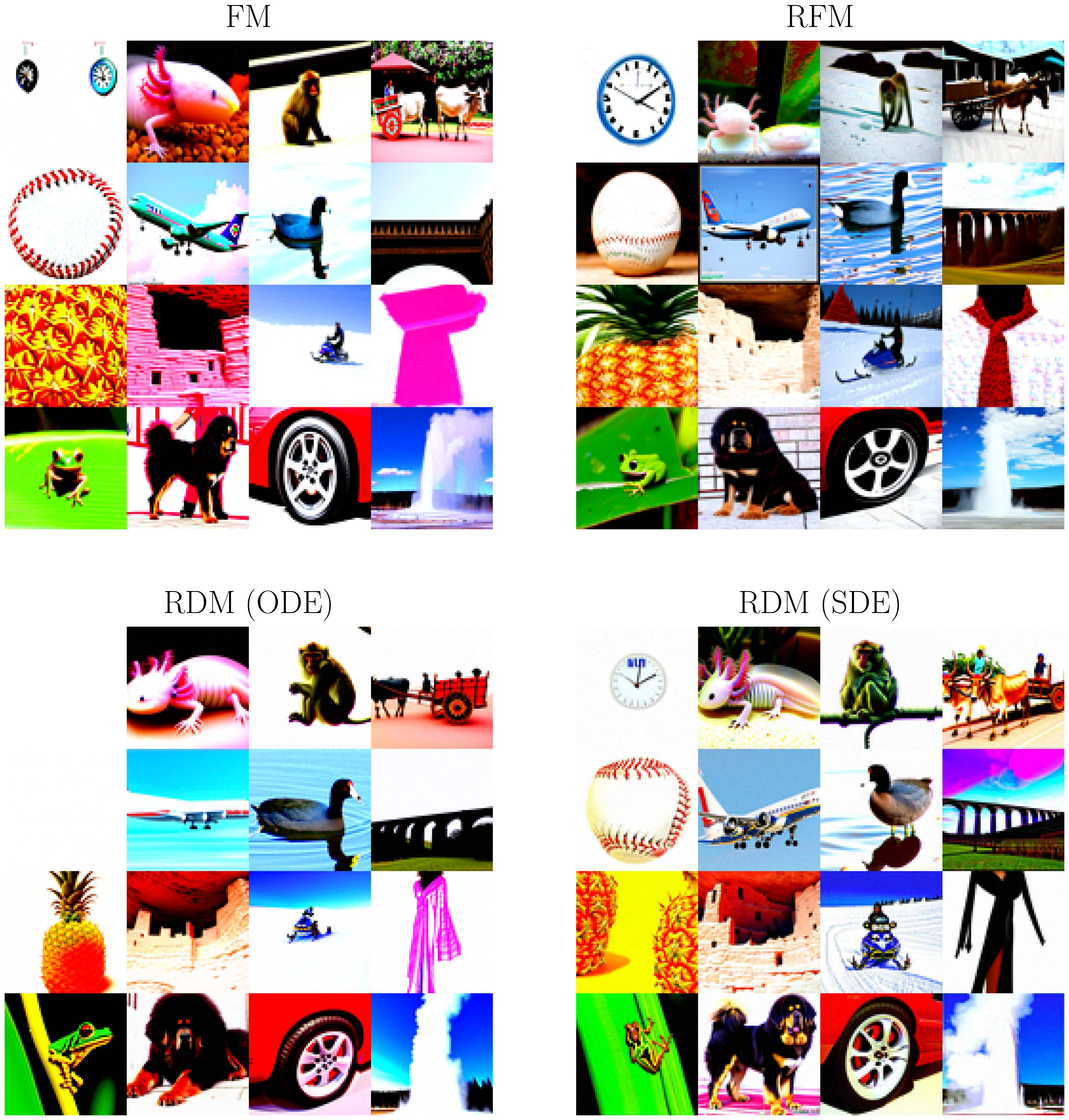}
    \caption{Class-conditioned guided samples with a high guidance weight $w=15$ on ImageNet ($64\times 64$) from different methods, including i) the CNFs trained with FM (upper left,  $\textrm{NFE}=769$, Dormand-Prince method); ii) the reflected CNFs trained with RFM (upper right, $\textrm{NFE}=723$, Dormand-Prince method); iii) backward ODE in RDM (lower left, $\textrm{NFE}=871$, Runge–Kutta–Fehlberg method without reflection); iv) backward SDE in RDM (lower right, $\textrm{NFE}=800$, Euler-Maruyama method with reflection).
    The samples of i) and iii) are clipped to $[0,255]$ after the ODE simulation.
    }
    \label{fig:imagenet-appendix}
\end{figure}

\end{document}